\documentclass[10pt,twocolumn,letterpaper]{article}

\usepackage[pagenumbers]{cvpr}

\usepackage{amsthm}
\newtheorem{definition}{Definition}

\usepackage{algorithm}
\usepackage{algorithmic}

\definecolor{cvprblue}{rgb}{0.21,0.49,0.74}
\usepackage[breaklinks,colorlinks,allcolors=cvprblue]{hyperref}

\def\paperID{Anonymous} % *** Enter the Paper ID here
\def\confName{CVPR}
\def\confYear{2026}

\title{\textsc{Scene2Demo}: Automated Embodied Data Generation\\for Robot Policy Learning}

\author{%
Xiang Liu$^{*}$ \qquad
Sen Cui$^{*,\dagger}$ \qquad
Guocai Yao$^{\dagger}$ \qquad
Zhong Cao \\[0.35em]
Jingheng Ma \qquad
Min Zhang \qquad
Changshui Zhang$^{\dagger}$ \\[0.55em]
{\small $^{*}$Equal contribution \qquad $^{\dagger}$Corresponding authors} \\[0.20em]
{\tt\small cuis@mail.tsinghua.edu.cn \qquad yaoguocai@pku.edu.cn \qquad zcs@mail.tsinghua.edu.cn}
}

\usepackage{multirow}
\usepackage{wrapfig}
 
\usepackage[most]{tcolorbox}
\newtcblisting{OutputBox}[1][]{
    enhanced,
    breakable,            % 允许跨页
    listing only,
    listing options={
        basicstyle=\ttfamily\small, % 等宽字体
        breaklines=true,            % 自动换行
        columns=fullflexible,
        alsoother={_},              % Allow long snake_case identifiers to wrap.
        literate={’}{{'}}1 {—}{{\textemdash}}1
    },
    colback=gray!10,      % 背景色：浅灰
    colframe=gray!50,     % 边框色：深灰
    arc=0pt,              % 直角边框（如需圆角改成 3mm）
    boxrule=1pt,          % 边框宽度
    title={#1},           % 标题（可选）
    coltitle=black,       % 标题文字颜色
    attach boxed title to top left={yshift=-2mm, xshift=2mm},
    boxed title style={colback=white, colframe=white}, % 标题框样式
    before skip=8pt,    % Space before each listing box.
    after skip=8pt,
}

\newcommand{\method}{\textsc{Scene2Demo}}

\begin{document}

\maketitle

\begin{abstract}
We present \method, an automated framework for offline embodied data generation. Given a single RGB scene image and a user query, \method~ constructs an interactive simulated scene and generates executable task configurations, primitive action flows, multi-view execution videos, and offline robot-learning datasets. Its five role-specialized generation modules perform task understanding, symbolic decomposition, BDDL configuration generation, object-size adjustment, and action-flow planning. The resulting object-centric task specifications and action transitions are organized as path search over a scene-conditioned object-action graph.
Failed or incomplete executions are refined by two feedback agents: a Robotic Safety Inspector that produces step-wise visual critiques and a Robotic Task Supervisor that revises action sequences or parameters using the rollout and previous evolution history. Across 102 automatically generated primitive scene-task pairs, \method~ achieves a 71.6\% execution success rate; on four representative long-horizon tasks, feedback-driven refinement improves both task success and subtask-level execution quality over primitive-only execution. Comparisons with RoboGen and GenSim2 show stronger execution performance under automated data-generation settings. Finally, behavior cloning policies achieve 96.0\% and 92.0\% success on two representative tasks, showing that the generated data supports downstream policy learning. Our project page is available at \url{https://scene2demo-anon.github.io/}.

\end{abstract}

\begin{figure*}[h]
	\vspace{-.4cm}
	\centering{
		\includegraphics[width=2\columnwidth]{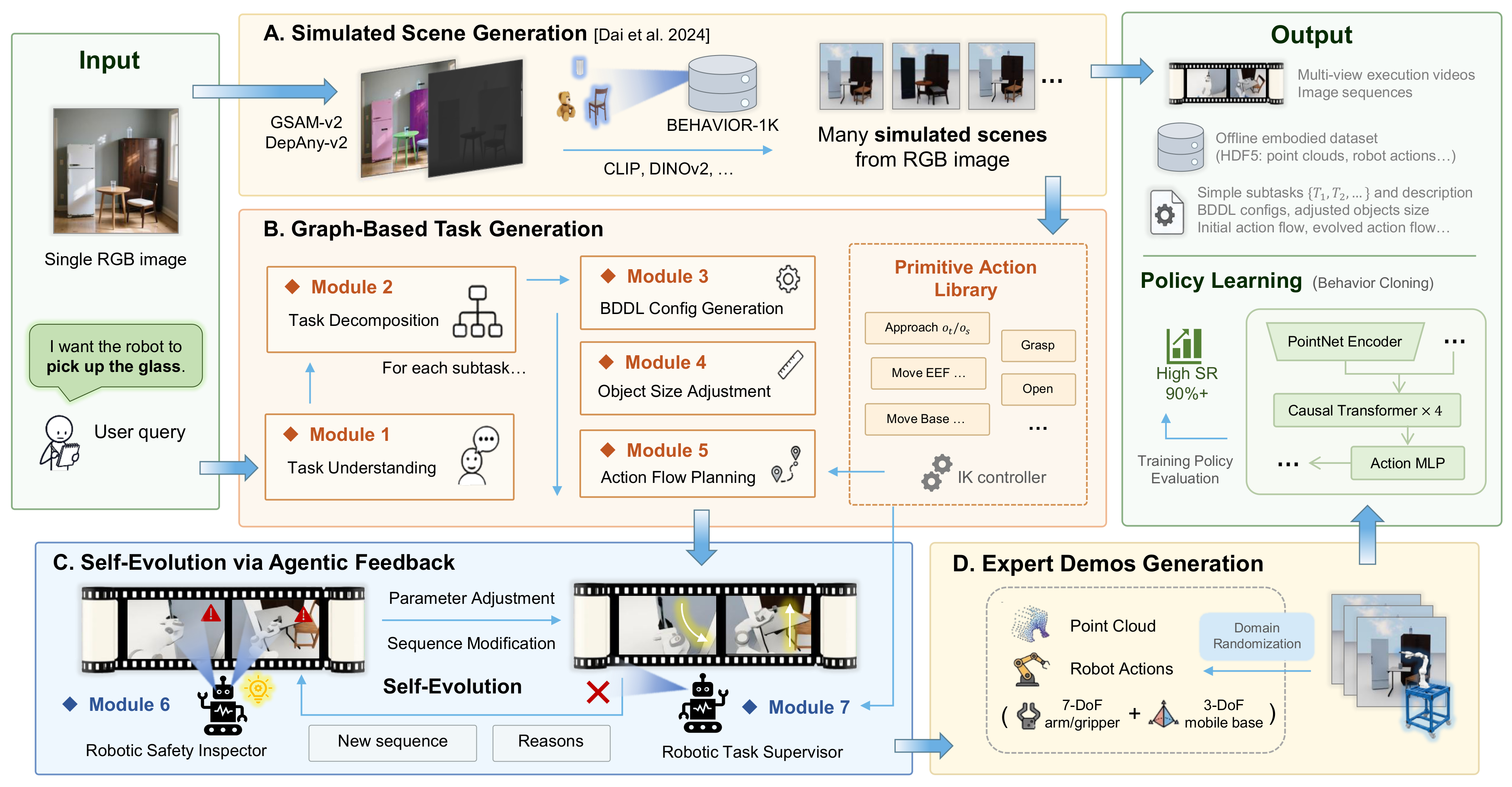}}
        \vspace{-.2cm}
	\caption{\textbf{Pipeline overview.} From one RGB scene image and a task query, \method~ reconstructs an interactive scene, generates symbolic task configurations and action flows, and refines failed executions using visual-feedback agents. Successful rollouts provide multi-view videos and offline data for policy learning.}
	\label{fig:intruduction}
	\vspace{-.4cm}
\end{figure*}

\section{Introduction}
% \vspace{-0.4cm}
Embodied robot learning requires executable, physically valid interaction data with a structure suitable for downstream training. We study offline embodied data generation: given an RGB scene image and a user task query, how can a system automatically produce executable robot task data? The desired outputs are an interactive simulated scene, task configurations, primitive action flows, multi-view videos, and observation-action datasets reusable for imitation learning, learning from demonstrations, and world modeling.

We propose \method, an automated framework for embodied data generation and robot policy learning. Following Digital Cousins~\cite{acdc}, we reconstruct an interactive scene from an RGB image. On this substrate, \method~ generates symbolic tasks, BDDL-style configurations~\cite{behavior1k}, object-size adjustments, and executable action flows, then refines unsuccessful executions through visual feedback.

The first stage uses five role-specialized modules: query expansion, symbolic decomposition, configuration generation, object-size adjustment, and initial action-flow planning. Inspired by automated data-generation systems such as RoboGen~\cite{robogen} and GenSim2~\cite{gensim2}, this organization exposes inspectable intermediate outputs and constrains decomposition with symbolic definitions. The object-action graph is the planning abstraction: nodes encode object-action-time configurations, and edges represent primitive transitions or transfers between subtasks.

The second stage refines failed or incomplete subtasks. A Robotic Safety Inspector examines primitive actions and recent multi-view frames to produce step-wise critiques. A Robotic Task Supervisor combines these critiques with the rollout, current action flow, and evolution history to revise sequences or parameters within the same primitive library. This feedback loop follows agentic workflows that combine reasoning, acting, feedback, and memory~\cite{react,reflexion,autogen}.

Our symbolic interface defines simple tasks through object roles and checkable initial and goal predicates. Long-horizon task generation becomes path search over a scene-conditioned object-action graph: intra-task edges form action flows, while transfer edges connect compatible subtask boundaries. These constraints provide stable, parseable task units. In our repeated-query study, \method~ consistently produces the same subtask structure, whereas unconstrained baselines yield varying subtask counts.

Across 102 generated primitive scene-task pairs, \method~ achieves 71.6\% execution success. Feedback-driven refinement improves complete-task success and subtask-level execution quality on four representative long-horizon tasks and outperforms RoboGen and GenSim2 in complete-task success under the evaluated automated-generation settings. Behavior cloning policies trained on 100 generated demonstrations achieve 96.0\% success on opening a refrigerator and 92.0\% on picking up a glass.

Our contributions are threefold: an automated framework that converts an RGB image and query into executable data for policy learning; a symbolic interface and scene-conditioned graph abstraction for stable task search; and step-wise visual inspection with task-level supervision to revise failed action flows.

\section{Related Works}
\paragraph{Generative video models and physics simulators.}
Generative video models have made rapid progress in visual prediction and interactive scene synthesis~\cite{videogpt,genie,sora}. However, recent studies suggest that video generation models can still produce physically inconsistent results and may fail to generalize physical laws beyond training-like scenarios~\cite{WorldModelBench,HowFarVideoPhys}. For robot interaction data, physics simulators remain a more direct substrate because they provide explicit states, collision handling, and action-conditioned transitions. Simulators and benchmarks such as iGibson~\cite{igibson2}, ThreeDWorld~\cite{threedworld}, Genesis~\cite{genesis}, ManiSkill~\cite{maniskill3}, OmniGibson~\cite{behavior1k}, and IsaacGym~\cite{isaacgym} support embodied interaction and robot learning in physically grounded environments. \method~ follows this simulation-based direction, but focuses on automatically generating executable task data rather than proposing a new simulator.

\paragraph{Automated scene, task, and demonstration generation.}
Recent work has explored using foundation models to reduce the manual effort of robot scene, task, and data generation. RoboGen generates tasks, scenes, rewards, and supervision for automated robot learning~\cite{robogen}, while GenSim2 scales robot data generation with multimodal and reasoning LLMs~\cite{gensim2}. Other works generate or reconstruct embodied simulation environments from language, real-world images, or 3D scans, including Holodeck, ReGen, GRS, EmbodiedGen, MetaScenes, and Digital Cousins~\cite{Holodeck,regen,grs,wang2025embodiedgengenerative3dworld,metascenes,acdc}. These works provide important tools for creating realistic or interactive simulation substrates. \method~ leverages interactive scene construction as the substrate for a multi-agent workflow that generates, executes, inspects, and revises embodied task data.

\paragraph{LLM/VLM agents and workflow-based reasoning.}
ReAct couples reasoning traces with environment actions~\cite{react}, Reflexion improves agents through verbal feedback and episodic memory~\cite{reflexion}, and AutoGen organizes multiple customizable agents through conversation-based workflows~\cite{autogen}. In robotics, SayCan grounds language planning with feasible robot skills~\cite{saycan}, Code-as-Policies represents robot policies as language-model-generated programs~\cite{codeaspolicies}, and CaP-X benchmarks Code-as-Policy agents that compose perception and control primitives~\cite{capx}. Recent robotic agent frameworks also explore using LLMs and VLMs for general robotic manipulation~\cite{maniagent}. \method~ follows this modular workflow perspective, but grounds it in offline embodied data generation, while the shared primitive action library serves as an executable interface.

\paragraph{Self-evolution, failure correction, and policy learning.}
Self-evolution in \method~ is related to feedback-driven robot execution and VLM-guided failure correction. Inner Monologue shows that LLM planners can incorporate environment feedback for embodied planning~\cite{inner_monologue}, while recent VLM-based systems detect, explain, or correct manipulation failures through visual reasoning, constraint monitoring, or visual-symbol guidance~\cite{aha,code_as_monitor,vifailback}. LLM/VLMs have also been used for reward design, task supervision, and self-evolving embodied agents~\cite{ma2024eureka,ma2024dreureka,SEEA-R1}. After self-evolution, \method~ stores successful demonstrations for downstream learning. This connects to demonstration-generation systems such as MimicGen, SkillMimicGen, and MoMaGen~\cite{mimicgen,SkillMimicGen,momagengen}, as well as behavior-cloning and visuomotor policy learning methods such as RT-1 and Diffusion Policy~\cite{rt1,diffusion_policy}.

\section{Method}
Given an RGB scene image $X_0$ and a user query $I$, \method~ reconstructs an interactive scene using the ACDC pipeline~\cite{acdc}. Five modules generate structured tasks and action flows; the Robotic Safety Inspector and Robotic Task Supervisor then refine unsuccessful executions.

\subsection{Scene-Conditioned Simulation Substrate}
\label{sec:Scene-Conditioned Simulation Substrate}

The simulation substrate reconstructs an interactive scene from a single RGB image. Let $\mathcal{S}$ denote the universal physical state space of the environment, where each element $s_\tau \in \mathcal{S}$ represents a possible relationship between objects at time $\tau$. Given the input image $X_0 \in \mathbb{R}^{H \times W \times 3}$, the scene construction pipeline outputs a simulated initial state $S_0 \subset \mathcal{S}$ as $S_0 \sim p(S \mid X_0; \epsilon_0)$, where $\epsilon_0$ denotes the reconstruction error inherited from the scene construction pipeline. 

Following the original reconstruction workflow, this stage extracts object identities, geometry, and spatial relations, retrieves matching interactive assets, and arranges them in OmniGibson. The resulting $S_0$ is therefore not only a visual replica, but an executable scene in which task predicates can be checked and action flows can be rolled out. Scene reconstruction and predicate examples are detailed in Appendices~\ref{app:Simulated Scene Generation Details} and~\ref{app:Subtask Instantiation}.

\subsection{Modules 2--3: Symbolic Task Interface}
\label{sec:Symbolic Task Interface}

This interface directly connects Module 2, which decomposes a user query into simple subtasks, and Module 3, which turns each subtask into initial and goal conditions.

\paragraph{Object roles.}
For each simple task, $o_t$ denotes the directly manipulated object. Support objects $o_{s_1}$ and $o_{s_2}$ specify its initial and desired relations, respectively. Either may be null when a task only changes the state of $o_t$ itself.

The initialization and termination states are defined as functions of these object roles:
\begin{equation}
    s_{\mathrm{init}}(o_t,o_{s_1}),\
s_{\mathrm{goal}}(o_t,o_{s_2})
\subseteq \mathcal{S}.
\end{equation}
And we write $S \models s_{\mathrm{init}}(T)$ or $S \models s_{\mathrm{goal}}(T)$ when the global state $S$ satisfies the corresponding BDDL-style predicate.

\paragraph{Simple task.}
A simple task is the basic unit produced by Module 2 and instantiated by Module 3.

\begin{definition}[Simple Task]
A simple task $T$ is defined as $T = \langle o_t, o_{s_1}, o_{s_2}, s_{\mathrm{init}}, s_{\mathrm{goal}} \rangle$.
\end{definition}

Module 2 is responsible for decomposing a user query into one or more simple tasks and assigning the object roles $(o_t,o_{s_1},o_{s_2})$ for each subtask. Module 3 then instantiates the logical boundary of each subtask by generating BDDL-style predicates for $s_{\mathrm{init}}$ and $s_{\mathrm{goal}}$.

\paragraph{Complex long-horizon task.}
Module 2 represents a long-horizon task as a parseable, ordered sequence of simple tasks.

\begin{definition}[Complex Long-Horizon Task]
A complex long-horizon task $\mathcal{T}$ is defined as $\mathcal{T} = \{T_1, T_2, \ldots, T_K\},\ K \geq 2$. For a single simple task, we also use $\mathcal{T}=\{T_1\}$.
\end{definition}

\begin{figure*}[!t]
	\centering{
		\includegraphics[width=1.8\columnwidth]{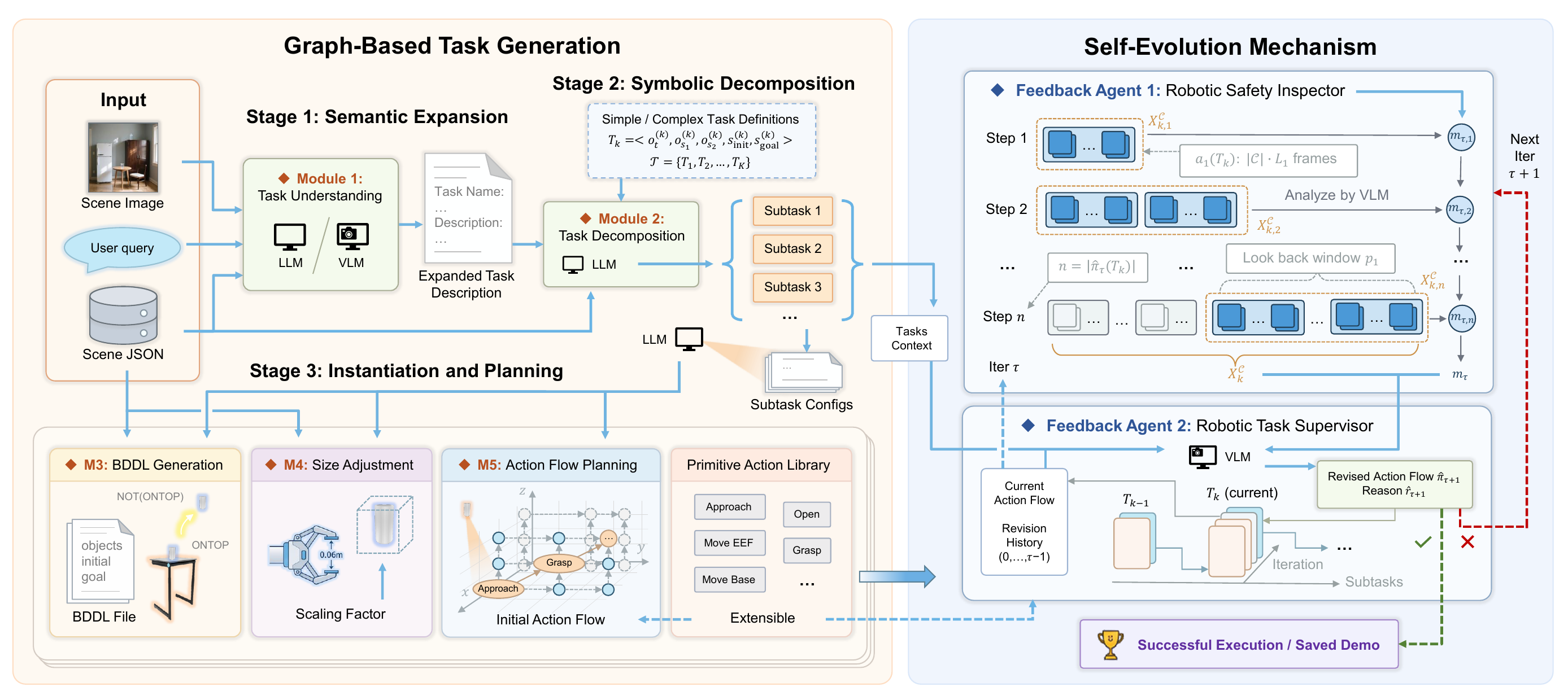}}
    \vspace{-0.4cm}
	\caption{\textbf{Structured task generation and feedback-driven refinement.} \textbf{Left:} five role-specialized modules expand the query, decompose it into symbolic subtasks, instantiate BDDL-style conditions, adjust object scale, and plan initial action flows. \textbf{Right:} the Robotic Safety Inspector provides step-wise visual critiques, and the Robotic Task Supervisor revises failed flows using the rollout and evolution history.}
	\label{fig:Graph-Based Task Generation with Self-Evolution}
	\vspace{0.1cm}
\end{figure*}

\subsection{Module 5: Action Flow and Action Transfer Planning}

Module 5 plans primitive-action programs within each simple task and transfers between consecutive subtasks.

\paragraph{Action flow for simple tasks.}
A simple task $T$ is executed by an action flow:
\begin{equation}
    \pi(T) = \{a_1(\theta_1 \mid T), a_2(\theta_2 \mid T), \ldots, a_n(\theta_n \mid T)\},
\end{equation}
where each $a_i$ is an atomic primitive selected from the primitive action library, and $\theta_i$ denotes its task-conditioned parameters, such as the target object, support objects, navigation targets, or continuous control values. The primitive action library contains navigation, base movement, end-effector operations, and other task-conditioned primitives, serving as the executable interface shared by the initial planner and the self-evolution stage.

Since the action flow is generated by an LLM/VLM module, we model it as a sample from a conditional distribution $\pi \sim p_F(\pi \mid T; \epsilon_1)$, where $\epsilon_1$ denotes generation error in action flow planning.

A simple task $T$ is executable if there exists at least one action flow $\pi$ such that $P(s_{\mathrm{goal}} \mid s_{\mathrm{init}}, \pi) > 0$. This condition means that the primitive sequence can move the subtask from its initial state to its goal state while preserving the irrelevant parts of the scene state.

\paragraph{Action transfer between subtasks.}
For a long-horizon task $\mathcal{T}=\{T_1,\ldots,T_K\}$, Module 5 connects adjacent subtasks under the global task context. We call their semantic and physical context shift an action transfer:
\begin{equation}
\begin{split}
    e_k =& e(T_k,T_{k+1}\mid \mathcal{T})\\
    =&
    \left\{
    a_{\mathrm{end}}(\theta_{\mathrm{end}}^{(k)} \mid \mathcal{T}),
    a_{\mathrm{start}}(\theta_{\mathrm{start}}^{(k+1)} \mid \mathcal{T})
    \right\}.
\end{split}
\end{equation}
Here, $a_{\mathrm{end}}$ and $a_{\mathrm{start}}$ are the terminal and initial actions of adjacent subtasks, with parameters conditioned on $\mathcal{T}$. A transfer need not add a physical operation; it encodes whether one subtask ends in a valid initialization of the next.

We model the action-transfer sequence $E=\{e_1,\ldots,e_{K-1}\}$ as
\vspace{-0.4cm}
\begin{equation}
    E \sim p_T(E \mid \mathcal{T}; \epsilon_2)
    =
    \prod_{k=1}^{K-1}
    p_T(e_k \mid \mathcal{T}, e_{<k}; \epsilon_2),
\vspace{-0.3cm}
\end{equation}
where $e_{<k}=\{e_1,\ldots,e_{k-1}\}$ denotes previous transfers. Conditioning on this history and the full task couples intra-task feasibility with globally consistent transitions.

\subsection{Object-Action Task Graph}
\label{sec:object_action_task_graph}

The symbolic task interface induces a directed graph for task planning. Let $\mathcal{O}$ denote the set of manipulable objects, $\mathcal{A}$ the primitive action library, and $\mathbb{N}^{+}$ the local temporal order inside a simple task. We use the universal template
\begin{equation}
    \mathcal{G}=(\mathcal{V},\mathcal{E}), \qquad
    \mathcal{V}=\mathcal{O}\times\mathcal{A}\times\mathbb{N}^{+},
\end{equation}
where a node $v=(o,a,n)$ represents applying action $a$ to object $o$ at local step $n$. Edges in $\mathcal{E}$ encode either feasible action transitions within a simple task or transfers between consecutive subtasks.

For a reconstructed scene $S_0$, Modules 2--5 instantiate a scene-conditioned subgraph $\mathcal{G}_{S_0}$ using the available objects, generated initial and goal predicates, supported primitives, and inter-task boundary conditions. Given the user query $I$, task generation searches this graph for an ordered path whose intra-task segments execute the simple tasks and whose transfer edges preserve the context of the long-horizon objective. Operationally, Module 2 proposes the subtask sequence, Module 3 supplies checkable symbolic boundaries, Module 4 filters size-incompatible object configurations, and Module 5 constructs the executable action-flow and transfer edges.

\begin{figure}[t]
	\centering
    \includegraphics[width=0.43\textwidth]{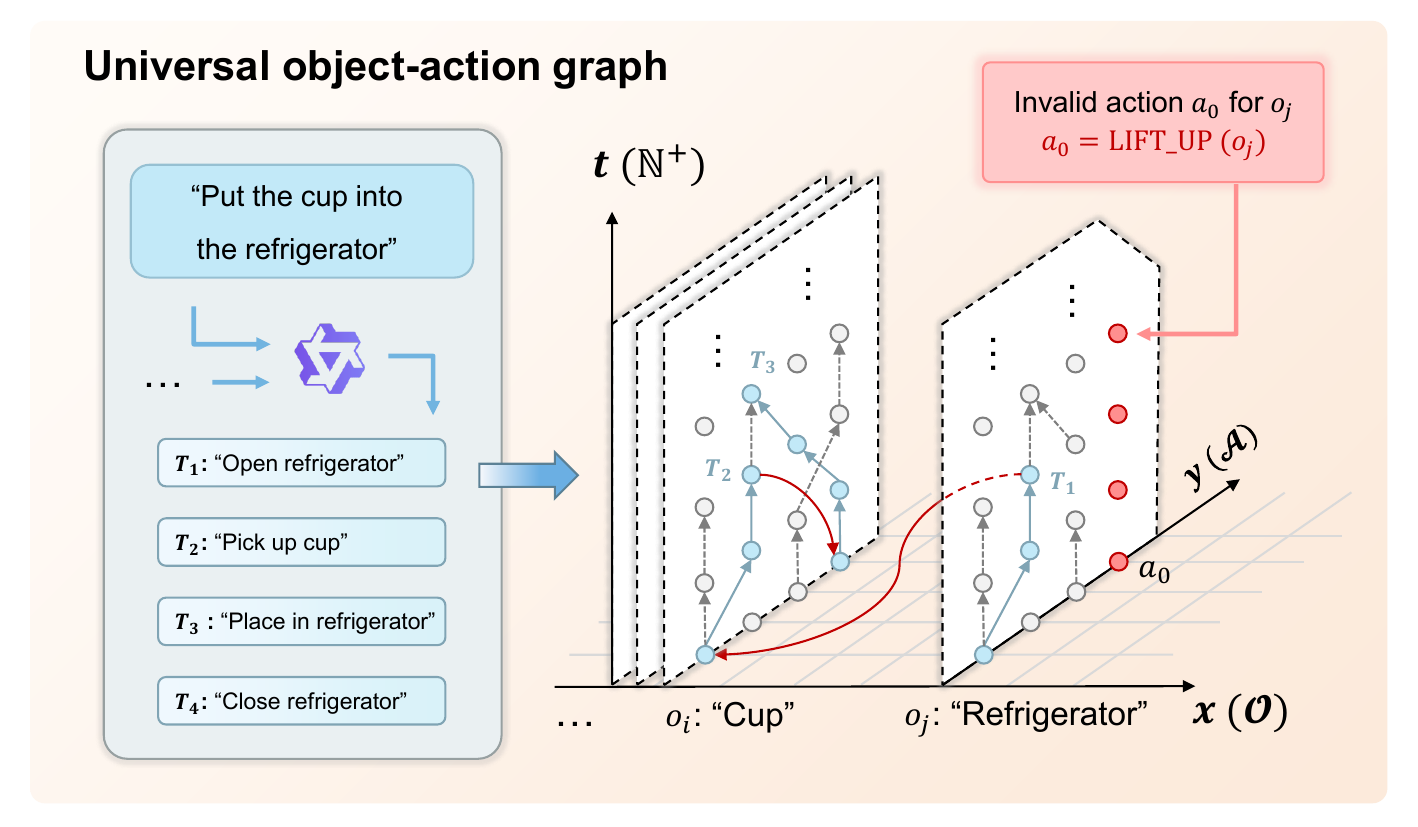}
	\vspace{-.3cm}
	\caption{\textbf{Object-action task graph.} Nodes represent object-conditioned primitive actions at local time steps. Intra-task edges form action flows, while transfer edges connect adjacent subtasks; task generation searches a scene-conditioned path from the initial state to the user goal.}
    \label{fig:object_action_task_graph}
	\vspace{0.1cm}
\end{figure}

\subsection{Structured Task Generation}

Given the user query $I$, the simulated scene $S_0$, and the original RGB image $X_0$, the task-generation workflow produces a structured task sequence $\mathcal{T} \sim p_G(\mathcal{T}\mid I,S_0,X_0;\epsilon_3)$, where $\epsilon_3$ denotes the uncertainty introduced by task understanding and decomposition.

The pipeline contains five role-specialized modules. Module 1 expands the user query into a scene-grounded task description that identifies the intended interaction and relevant objects. Module 2 decomposes the expanded task into simple tasks according to the symbolic definitions above. For each simple task, Module 3 generates a BDDL-style configuration that defines $s_{\mathrm{init}}$ and $s_{\mathrm{goal}}$. Module 4 adjusts the scale of manipulable objects when the reconstructed asset size is incompatible with the robot gripper. Module 5 maps each simple task to an initial action flow and checks the inter-task transfers between consecutive subtasks. Separating these responsibilities makes intermediate outputs inspectable and allows later feedback to revise the action program without regenerating the scene or task specification.

Module 4 preserves the scale of fixed fixtures, while reducing oversized manipulable objects to a feasible grasping width. Module 5 selects from 21 primitives spanning navigation, base motion, end-effector motion, grasp/release, and articulated-object interaction. Structured object references and continuous parameters make the same action representation both executable and editable by the feedback agents. Appendix~\ref{app:task_generation_details} gives the module interfaces and primitive library.

\subsection{Self-Evolution for Task Progression}

The initial action flow may fail due to imperfect scene reconstruction, object-size mismatch, navigation offsets, or accumulated errors in long-horizon execution. To handle these failures as much as possible, \method~ introduces a self-evolution mechanism.

\paragraph{Robotic Safety Inspector.}
At evolution iteration $\tau$, for each subtask $T_k\in \mathcal{T}$, the current action flow
$\hat{\pi}_\tau(T_k)$ is executed in simulation, and multi-view video frames are collected as visual feedback. Let
$\mathcal{C}\subset\{\mathrm{Global},\mathrm{Head},\mathrm{Wrist}\}$ denote the set of camera views, and let $p_2$ be the frame sampling interval. Let $x_{k,j}^{c,l}\in \mathbb{R}^{H\times W\times 3}$, if action $a_j$ takes time $t_j$, the image sequence for the $j$-th action from view $c$ is
\begin{equation}
    X_{k,j}^{c}
    =
    \left(
    x_{k,j}^{c,1},
    \ldots,
    x_{k,j}^{c,L_j}
    \right),
    \quad L_j=\min\left(\left\lceil \frac{t_j}{p_2} \right\rceil, 6\right).
\end{equation}

The Robotic Safety Inspector analyzes execution locally. To provide temporal context, we use a look-back window $p_1$. For the $j$-th action, the inspector observes the multi-view frames corresponding to actions
$\{a_{\max(j-p_1,1)},\ldots,a_j\}$. Let 
\begin{equation}
    X_{k,j}^{\mathcal{C}}
    =
    \left\{
    X_{k,\max(j-p_1,1)}^{c},
    \ldots,
    X_{k,j}^{c}
    \right\}_{c\in \mathcal{C}},
\end{equation}
and $X_k^{\mathcal{C}}=\mathop{\textstyle\bigcup}\nolimits_{j=1}^{|\hat{\pi}_{\tau}(T_k)|}X_{k,j}^{\mathcal{C}} $ denote the complete visual observation for the subtask. The inspector outputs a step-wise critique $m_{\tau,j} \sim p_E(m \mid X_{k,j}^{\mathcal{C}};\epsilon_4)$, where $\epsilon_4$ denotes the uncertainty of VLM-based visual diagnosis. 

\paragraph{Robotic Task Supervisor.}
At evolution iteration $\tau$, let $m_\tau=\{m_{\tau,j}\}_{j=1}^{|\hat{\pi}_\tau(T_k)|}$ denote the sequence of step-wise critiques for the current action flow $\hat{\pi}_\tau(T_k)$. 
We define the evolution history before iteration $\tau$ as 
$
    \mathcal{H}_{<\tau}
    =
    \{(\hat{\pi}_i(T_k), m_i, \hat{r}_{i+1})\}_{i=0}^{\tau-1},
$
where each record stores a previously evaluated action flow, its step-wise critiques, and the supervisor's revision reason.

The Robotic Task Supervisor aggregates the current action flow, the current critiques, the visual rollout, the task context, and the previous evolution history to produce a revised action flow:
\begin{equation}
    (\hat{\pi}_{\tau+1}(T_k), \hat{r}_{\tau+1})
    \sim
    p_{\mathrm{sup}}\left(
    \cdot
    \mid
    \hat{\pi}_{\tau}(T_k),
    m_\tau,
    X_k^{\mathcal{C}},
    \mathcal{T},
    \mathcal{H}_{<\tau};
    \epsilon_4
    \right),
\end{equation}
where $\hat{r}_{\tau+1}$ denotes the supervisor's global textual explanation for the revision.
The revised action flow $\hat{\pi}_{\tau+1}(T_k)$ may change the primitive sequence, adjust action parameters, or modify navigation waypoints. 
After each failed attempt, we update the history as
$
    \mathcal{H}_{<\tau+1}
    =
    \mathcal{H}_{<\tau}
    \cup
    \{(\hat{\pi}_\tau(T_k), m_\tau, \hat{r}_{\tau+1})\},
$
which helps the supervisor avoid repeatedly trying previously failed configurations. 
Before evaluating a revised flow, we reset the simulator to the saved initial state of the current subtask, so that a failed rollout does not alter the starting conditions of the next attempt. The loop repeats until $s_{\mathrm{goal}}(T_k)$ is satisfied or the iteration limit is reached. Appendix~\ref{app:self_evolution_details} gives the complete protocol and algorithm.

\section{Experiments}
We evaluate \method~ as an automated embodied data generation engine. Our experiments answer: 
(1) whether the generated primitive scene-task pairs are executable, 
(2) whether self-evolution improves complex long-horizon task execution, 
(3) how \method~ compares with existing automated data-generation pipelines, and 
(4) whether the generated demonstrations provide useful data for downstream policy learning.

\subsection{Experimental Setup}

\textbf{Implementation details.} All simulated scenes are built on the OmniGibson platform~\cite{behavior1k}, which provides physical simulation, object states, and rendering for the reconstructed interactive environments. We use a \textit{FrankaMounted} robotic arm with a mobile base, perceived through on-board RGB-D sensors and controlled by an inverse-kinematics (IK) controller; the robot is uniformly rescaled by a factor of $s=0.7$ to match the scale ambiguity of monocular reconstruction, with the grasping offset set to $0.14\times s$. Unless otherwise specified, all generation and self-evolution modules use the open-source VLM Qwen3-VL (\textit{qwen3-vl-235b-22a-thinking}), and $p_1=1$ for the look-back window. The three camera views are Global, Head, and Wrist.

We first evaluate primitive task generation on 102 scene-task pairs reconstructed from 34 RGB scene images generated by Qwen3-Max and Seedream-4.5, with three asset variations per image. 
These tasks include articulated-object operations ($\mathcal{T}_5$: open/close) and rigid-object manipulation ($\mathcal{T}_6$: pick/place). We then evaluate long-horizon execution on four representative tasks: 
$\mathcal{T}_1$: ``Put the glass into the refrigerator'',
$\mathcal{T}_2$: ``Put the apple into the refrigerator'',
$\mathcal{T}_3$: ``Put the apple into the bowl'',
$\mathcal{T}_4$: ``Put the cup on the table''.

\begin{table}[h]
    \vspace{-0.2cm}
    \centering
    \caption{\textbf{Primitive-task execution.} Success rates on 102 scene-task pairs from 34 generated RGB scene images with three asset variations. $\mathcal{T}_5$ denotes Open/Close and $\mathcal{T}_6$ denotes Pick/Place.}
    \vspace{-0.2cm}
    \label{tab:primitive_tasks}
\resizebox{0.35\textwidth}{!}{
    \begin{tabular}{lccc}
    \toprule
    Category & Count & Success & SR (\%) \\
    \midrule
    $\mathcal{T}_5:$ Open/Close & 36 & 24 & 66.7 \\
    $\mathcal{T}_6:$ Pick/Place & 66 & 49 & 74.2 \\
    \midrule
    Total & 102 & 73 & 71.6 \\
    \bottomrule
    \end{tabular}
}
    \vspace{-0.2cm}
\end{table}

\begin{table*}[!t]
\centering
\small
\vspace{0.2cm}
\caption{
\textbf{Main execution comparison.} Each cell reports \textbf{SR / Subtask-Level Score} (\%). M1/M2 are \method~with primitive-only planning and with refinement; M3/M4 are GenSim2 and RoboGen; M5 is RoboGen with task-specific RL. Per-model refinement results are given in Table~\ref{tab:self_evolution_detailed}.
}
\vspace{-0.2cm}
\label{tab:main_comparison}
\resizebox{0.7\textwidth}{!}{
\begin{tabular}{lcccccc}
\toprule
Method 
& $\mathcal{T}_1$ & $\mathcal{T}_2$ & $\mathcal{T}_3$ & $\mathcal{T}_4$ & $\mathcal{T}_5$ & $\mathcal{T}_6$ \\
\midrule
\textbf{M1: Ours (Primitive)} 
& 0 / 50 & 0 / 50 & 40 / 70 & 0 / 50 & 66.7 / 66.7 & \textbf{74.2} / 74.2 \\
\midrule
\textbf{M2: Ours (Prim+Evo)}
& \textbf{40} / \textbf{84} & \textbf{30} / \textbf{79} & \textbf{70} / \textbf{93} & \textbf{70} / \textbf{85} & -- & -- \\
\midrule
M3: GenSim2 (Primitive)
& 0 / 73 & 5 / 78 & 55 / 91 & 0 / 68 & \textbf{100} / \textbf{100} & 38.8 / \textbf{82.1} \\
\midrule
M4: RoboGen (Primitive)
& 0 / 50 & 0 / 46.5 & 0 / 50 & 0 / 50 & 0 / 50.2 & 0 / 42.8 \\
\midrule
M5: RoboGen (Prim+RL)
& fail & fail & fail & fail & success & success \\
\bottomrule
\end{tabular}
}
\end{table*}

\begin{figure*}[!t]
	\vspace{-0.2cm}
	\centering
	\includegraphics[width=0.89\textwidth]{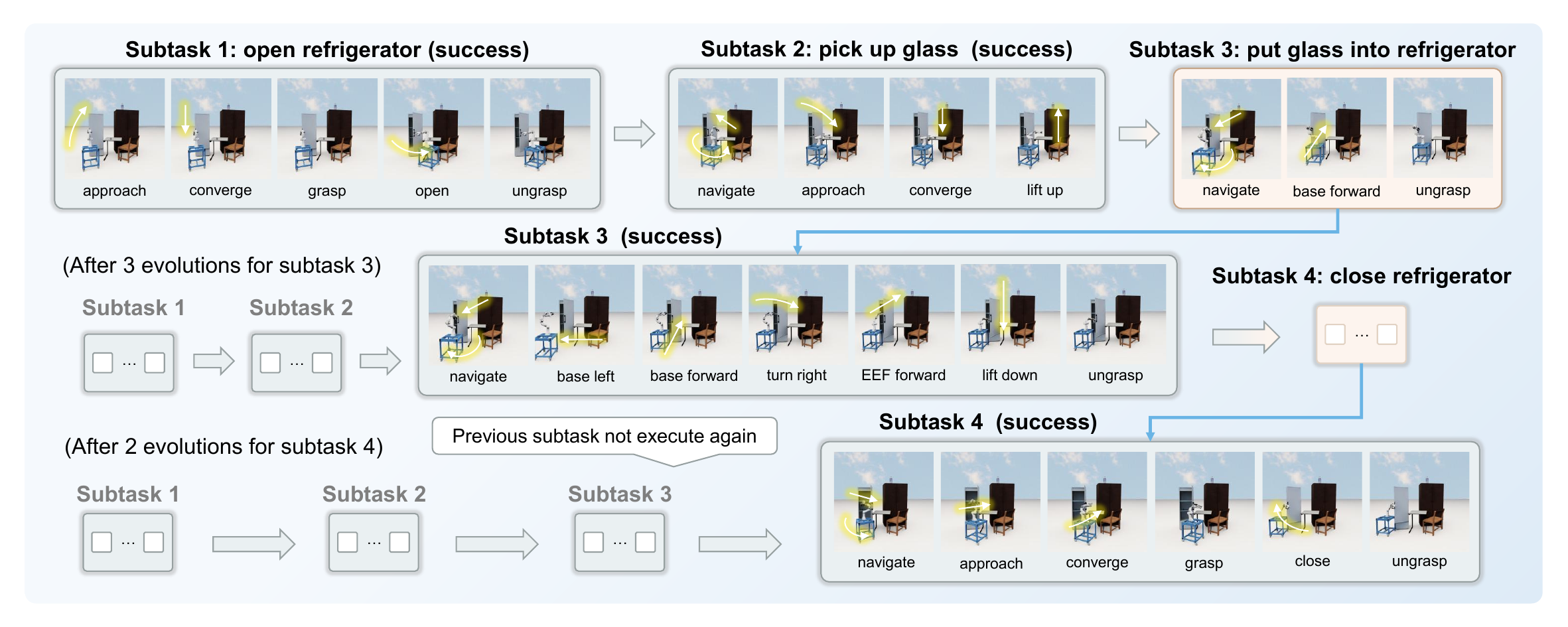}
    \vspace{-0.3cm}
	\caption{\textbf{Feedback-driven refinement on $\mathcal{T}_1$ (Glass$\to$Refrigerator).} The four subtasks are opening the refrigerator, picking up the glass, placing it inside, and closing the refrigerator.}
	\label{fig:Detailed Action Display of Self-Evolution}
	\vspace{0.1cm}
\end{figure*}

As summarized in Table~\ref{tab:primitive_tasks} and Figure~\ref{fig:combined_results}(a), our method achieves a 71.6\% overall success rate, with 66.7\% success on articulated-object operations and 74.2\% success on rigid-object manipulation. This result verifies that the generated task configurations and primitive action flows are executable across diverse reconstructed scenes.

\subsection{Task Execution and Baseline Comparison}

We report the main metric as \textbf{SR / Subtask-Level Score}. 
SR measures complete-task success, while the subtask-level score provides a diagnostic measure of partial execution quality under each system's own subtask definition and evaluator. Appendix~\ref{app:Additional Details for Main Task Execution Comparison} gives evaluation counts and metric definitions.

Table~\ref{tab:main_comparison} shows three main findings. 
First, self-evolution consistently improves long-horizon execution over primitive-only planning. 
Second, compared with RoboGen and GenSim2 under automated data-generation settings, \method~ achieves stronger complete-task success on the four representative long-horizon tasks, while also showing competitive subtask-level execution quality. 
Third, on simple task categories, \method~ remains competitive: it achieves 66.7\% on $\mathcal{T}_5$ and 74.2\% on $\mathcal{T}_6$, while GenSim2 is stronger on Open/Close but weaker on Pick/Place.

\subsection{Detailed Self-Evolution Performance}

To further analyze the role of self-evolution, we retain the per-task detailed results on long-horizon tasks, including task success rate (SR), evolution success rate (ESR), and the average number of iterations for successful executions. For $\mathcal{T}_1$ and $\mathcal{T}_3$, we compare Gemini-3-Flash, GPT-5.2, and Qwen3-VL; for others, we use Qwen3-VL. We set a strict failure condition: if any subtask exceeds 5 evolution iterations, the entire episode is marked as failed.

\begin{table*}[!t]
\centering
\vspace{0.2cm}
\caption{\textbf{Refinement performance on long-horizon tasks.} We report execution success (SR), failed first attempts rescued by refinement (ESR), and mean iterations for successful executions (Iter). $\mathcal{T}_3$/$\mathcal{T}_4$ have two subtasks (--).}
\vspace{-.2cm}
\resizebox{0.8\textwidth}{!}{
\begin{tabular}{c|c|c|cc|ccc|ccc|ccc|cc}
\toprule
\multirow{2}{*}{\textbf{Tasks}} & \multirow{2}{*}{\textbf{VLM}} & \multirow{2}{*}{\textbf{Demos}} & \multicolumn{2}{c|}{\textbf{Subtask 1}} & \multicolumn{3}{c|}{\textbf{Subtask 2}} & \multicolumn{3}{c|}{\textbf{Subtask 3}} & \multicolumn{3}{c|}{\textbf{Subtask 4}} & \multicolumn{2}{c}{\textbf{Complete Task}} \\
& & & SR & Iter & SR & ESR & Iter & SR & ESR & Iter & SR & ESR & Iter & SR & Iter  \\
\midrule
\multirow{3}{*}{$\mathcal{T}_1$} 
& Gemini-3-Flash & 10 & \textbf{100} & 0 & \textbf{100} & -- & 0 & 30 & 30 & 1.7 & \textbf{100} & \textbf{100} & 0.7 & 30 & 2.3 \\
& GPT-5.2 & 10 & \textbf{100} & 0 & \textbf{100} & -- & 0 & 30 & 30 & 2 & 66.7 & 66.7 & 1.0 & 20 & 3.5 \\
& \textbf{Qwen3-VL} & 20 & \textbf{100} & 0 & \textbf{100} & -- & 0 & \textbf{45} & \textbf{45} & 2.7 & 88.9 & 66.7 & 0.5 & \textbf{40} & 3.1 \\
\midrule
$\mathcal{T}_2$ & Qwen3-VL & 10 & \textbf{100} & 0 & \textbf{100} & \textbf{100} & 0.3 & 40 & 40 & 1.8 & 75 & 66.7 & 0.7 & 30 & 2.3\\
\midrule
\multirow{3}{*}{$\mathcal{T}_3$} 
& \textbf{Gemini-3-Flash} & 10 & \textbf{100} & 0 & \textbf{70} & \textbf{62.5} & 1.6 & -- & -- & -- & -- & -- & -- & \textbf{70} & 1.6 \\
& GPT-5.2 & 10 & \textbf{100} & 0 & 50 & 16.7 & 0.2 & -- & -- & -- & -- & -- & -- & 50 & 0.2 \\
& Qwen3-VL & 10 & \textbf{100} & 0 & 50 & 16.7 & 0.4 & -- & -- & -- & -- & -- & -- & 50 & 0.4 \\
\midrule
$\mathcal{T}_4$ & Qwen3-VL & 10 & \textbf{100} & 0 & 70 & 60 & 1.0 & -- & -- & -- & -- & -- & -- & 70 & 1.0 \\
\bottomrule
\end{tabular}
}
\vspace{0.2cm}
\label{tab:self_evolution_detailed}
\end{table*}

In addition to the main success-rate improvements, Table~\ref{tab:self_evolution_detailed} reveals that self-evolution is most useful for later subtasks, where long-horizon execution accumulates navigation and placement errors. 
For example, later subtasks such as transporting and precise placing often require multiple correction iterations, while early subtasks such as opening or pickup are usually solved with initial planning. 
This observation is consistent with the qualitative evolution examples in Figure~\ref{fig:Detailed Action Display of Self-Evolution}. Performance also depends on the feedback model: on $\mathcal{T}_3$, Gemini-3-Flash achieves 70\% complete-task success, compared with 50\% for Qwen3-VL and GPT-5.2.

Failures in the stages mainly stem from situations such as objects falling or getting stuck at the edges, and these situations are usually difficult to correct within a few steps.

\begin{figure*}[t]
    \centering
    \vspace{-0.2cm}

    \begin{minipage}[t]{0.23\textwidth}
        \vspace{0pt}
        \centering
        \includegraphics[width=\linewidth]{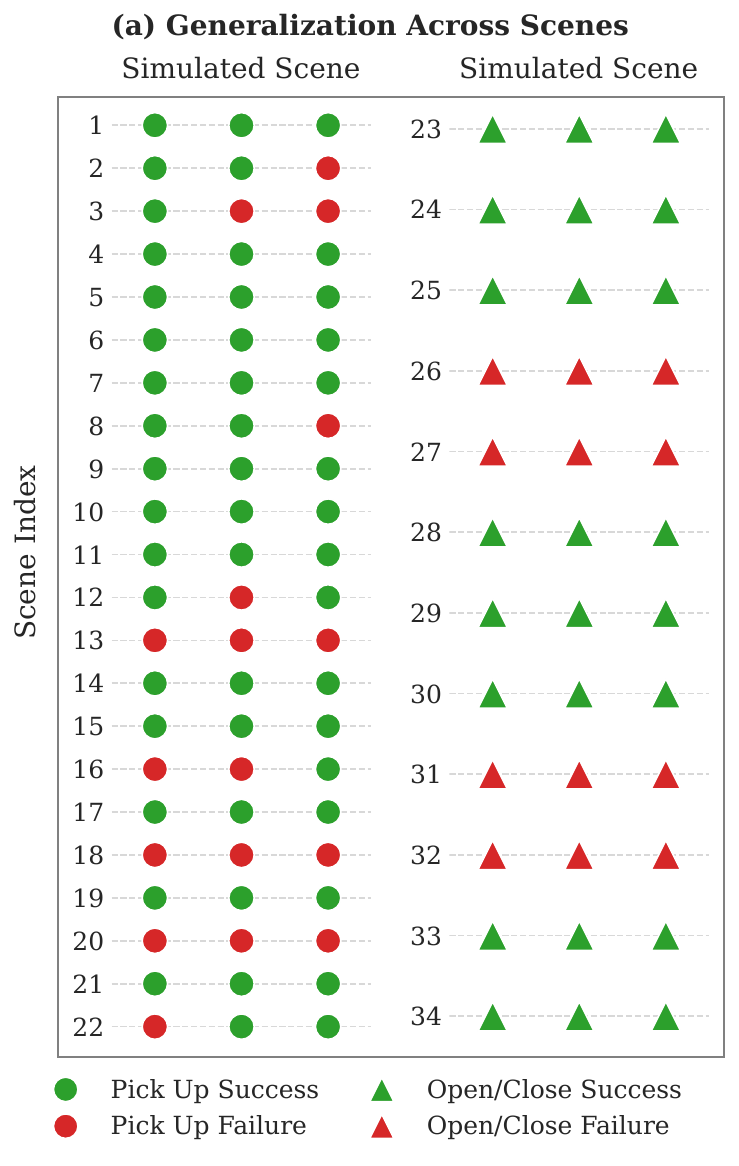}
    \end{minipage}
    \hspace{0.01\textwidth}
    \begin{minipage}[t]{0.565\textwidth}
        \vspace{0pt}
        \centering

        \begin{minipage}[t]{0.49\linewidth}
            \centering
            \includegraphics[width=\linewidth]{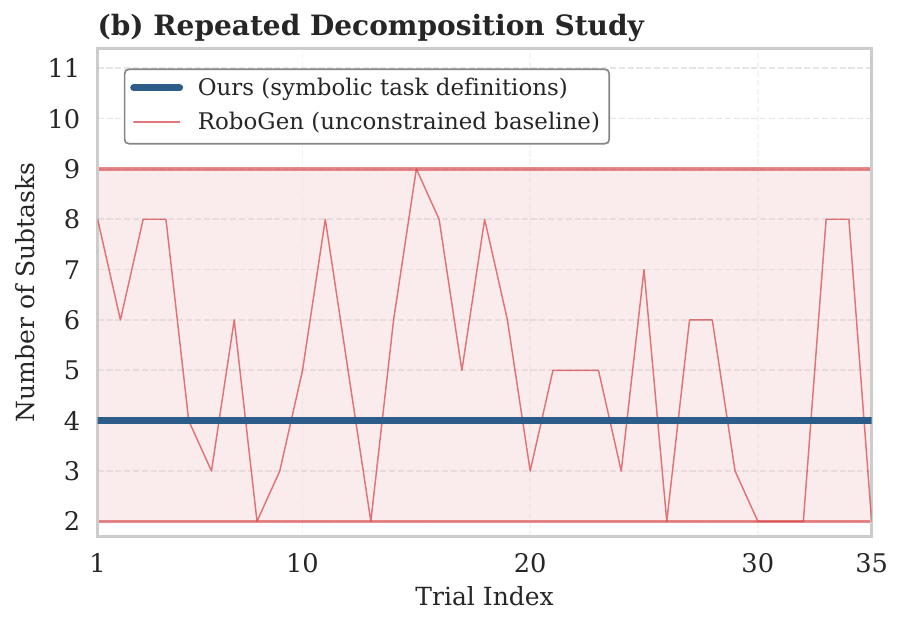}
        \end{minipage}
        \hfill
        \begin{minipage}[t]{0.49\linewidth}
            \centering
            \includegraphics[width=\linewidth]{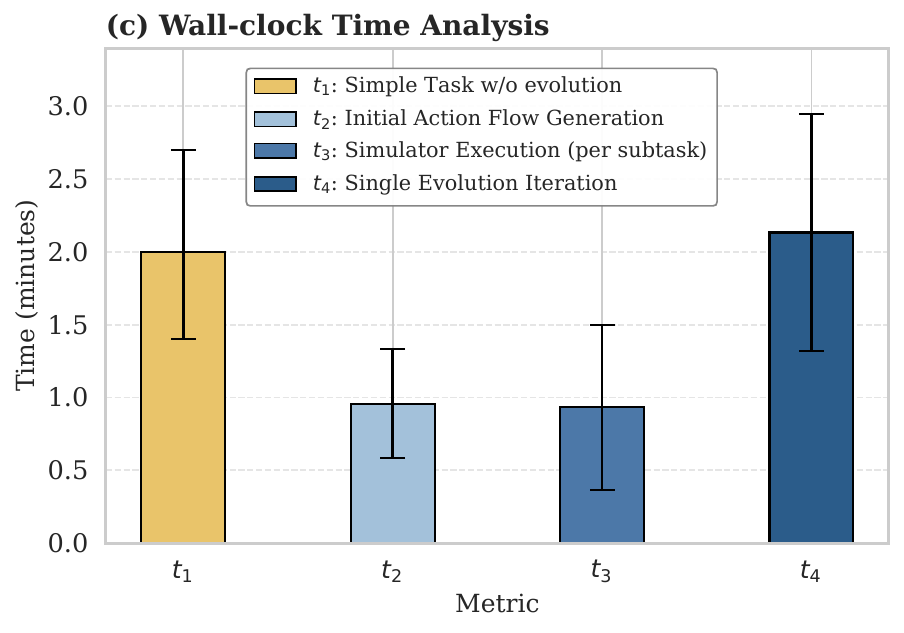}
        \end{minipage}

        \vspace{0.12cm}

        \begin{minipage}[t]{0.49\linewidth}
            \centering
            \includegraphics[width=\linewidth]{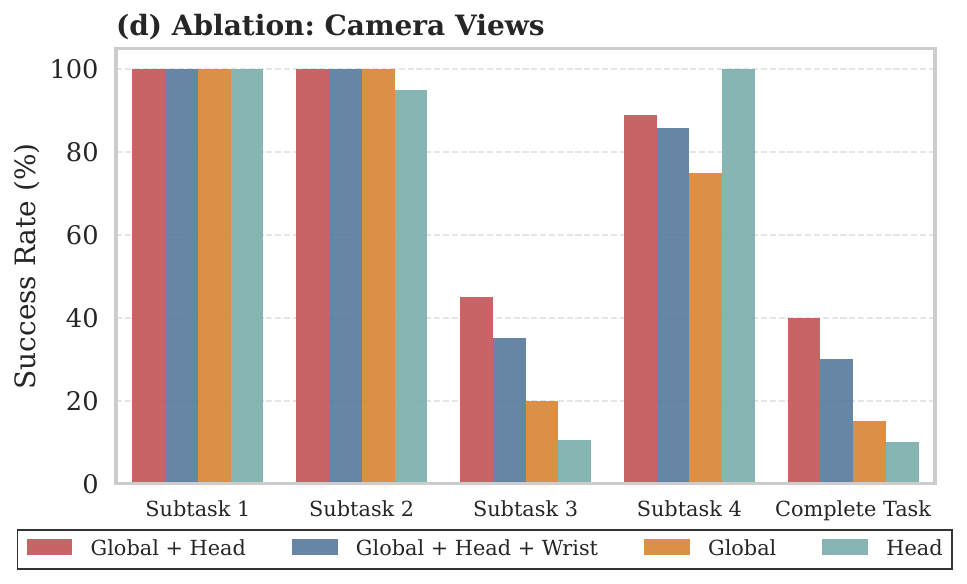}
        \end{minipage}
        \hfill
        \begin{minipage}[t]{0.49\linewidth}
            \centering
            \includegraphics[width=\linewidth]{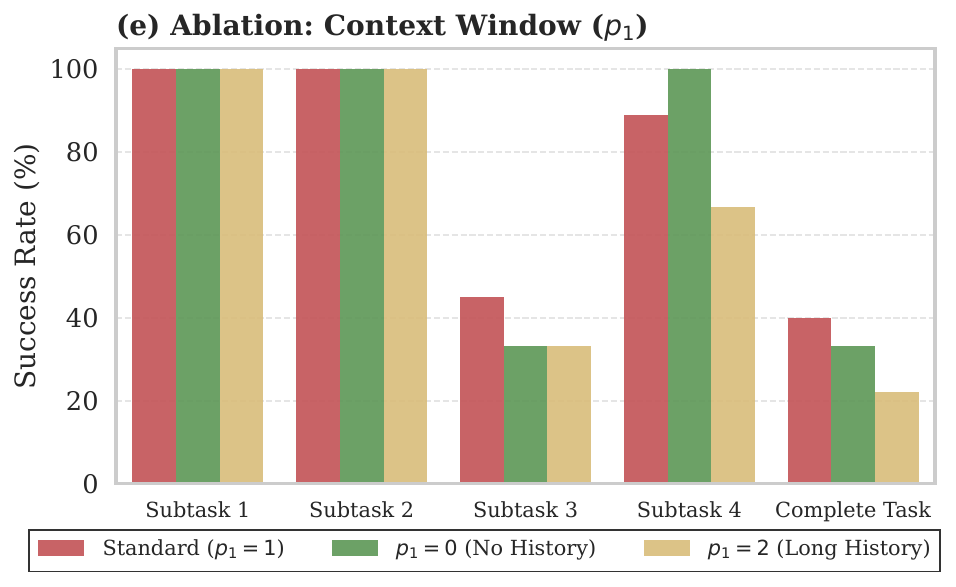}
        \end{minipage}
    \end{minipage}

    \vspace{-0.1cm}
    \caption{
    \textbf{Evaluation and ablations.}
    \textbf{(a)} Execution over 102 primitive scene-task pairs.
    \textbf{(b)} Repeated decomposition of one query.
    \textbf{(c)} Wall-clock time (mean and observed range).
    \textbf{(d--e)} Camera-view and context-window ablations on $\mathcal{T}_1$.
    }
    \label{fig:combined_results}
    \vspace{-0.1cm}
\end{figure*}

\subsection{Repeated Decomposition and Ablation Studies}

We also analyze the design choices and generation cost of \method. 
First, we study repeated decomposition consistency. As shown in Figure~\ref{fig:combined_results}(b), for the same instruction ($\mathcal{T}_2$: ``put the apple into the refrigerator''), \method~ consistently decomposes the task into the same executable subtasks across repeated runs, while unconstrained baselines produce varying subtask counts. 
This result supports our claim that symbolic task definitions improve decomposition stability.

Second, we analyze the wall-clock time of \method. 
As shown in Figure~\ref{fig:combined_results}(c), a simple task without self-evolution takes about 2 minutes. 
For long-horizon tasks, initial action-flow generation takes 35s--1m20s, simulator execution takes 22s--1m30s per subtask, and a single evolution iteration takes 1m19s--2m57s using Qwen3-VL. 
A complete long-horizon demonstration typically takes about 15--30 minutes including all iterations. 
Because \method~ is an offline data generation engine, the full refinement loop only needs to find a successful action configuration once, after which domain randomization can rapidly generate diverse demonstrations.

Third, we conduct ablation studies on visual feedback and temporal context. 
As shown in Figure~\ref{fig:combined_results}(d)(e), we investigate the impact of camera-view combinations and the look-back window $p_1$ used in the refinement mechanism.
The camera-view ablation evaluates different subsets of $\{\text{Global}, \text{Head}, \text{Wrist}\}$, and the context-window ablation varies the amount of recent action history. 
Consistent with the original findings, the standard and full multi-view settings achieve the best overall performance, while using only a single view leads to more failures. 
We also observe that increasing the look-back window from $p_1=1$ to $p_1=2$ does not improve performance and may even degrade complete-task success, while also increasing VLM input length and inference latency. Therefore, we use $p_1=1$.

\subsection{Downstream Policy Learning}

Finally, we evaluate whether the generated demonstrations can support downstream policy learning. 
We train behavior cloning policies using 100 expert demonstrations autonomously generated by \method, with randomization on target-object scale and rotation to improve data diversity. 
The policy uses a PointNet encoder for point clouds, fuses the point-cloud features with proprioceptive states, and applies a 4-layer Causal Transformer followed by an action MLP to predict 10-DoF actions (7-DoF arm/gripper and 3-DoF mobile base). Point clouds are downsampled to 1024 points and augmented with translation and Gaussian noise; the temporal policy uses a window of 10 observations. Policies are trained with behavior cloning for 200 epochs on one NVIDIA A800 GPU using AdamW. During data generation and evaluation, target-object scale is perturbed by $\pm0.10$ along each axis, and articulated objects additionally receive a rotation perturbation of up to $\pm\pi/12$ around the vertical axis.

On two representative tasks, the learned policies achieve 96.0\% success (24/25) on \emph{Open Refrigerator} and 92.0\% success (23/25) on \emph{Pick Up Glass}. 
This experiment evaluates the utility of the generated data for imitation learning rather than a new policy-learning method. Architecture, training, and randomization details appear in Appendix~\ref{app:Downstream Policy Learning}.

\section{Conclusion}
We introduced \method, an automated framework that converts an RGB image and a user query into an interactive simulated scene, executable task configurations, and offline robot-learning datasets. Its structured symbolic interface represents object-centric subtasks and primitive action transitions as a scene-conditioned graph-search problem. Role-specialized modules expose intermediate task specifications, boundary conditions, size adjustments, and action flows, while visual-feedback agents revise failed executions without rebuilding the full pipeline. The resulting demonstrations support downstream imitation learning, showing the value of \method~ as an offline data generation engine.

There are several limitations currently. First, \method~ is constrained by the capability of the primitive action library and the IK-based controller. The current action space covers Open/Close, Pick/Place, and mobile-base navigation, but not dexterous tasks. Future work may introduce stronger RL-based agents or learned low-level controllers to expand the action library and handle more complex manipulation skills. Second, our scene construction stage relies on the Digital Cousins pipeline~\cite{acdc}, so failures may come from imperfect scene reconstruction or missing interactive assets. Future work can explore stronger scene reconstruction and asset-grounding methods to improve the robustness of the simulation substrate. Third, our self-evolution mechanism is still an early closed-loop implementation and depends heavily on the visual reasoning ability of the underlying VLM. A promising direction is to build a skill documentation library from successful tasks, where reusable correction strategies and action patterns are stored and retrieved to diagnose and repair future failures.

\label{sec:maintext_end}
\clearpage
{
    \small
    \bibliographystyle{ieeenat_fullname}
    \bibliography{main}
}

% APPENDIX
\clearpage
\appendix
\section*{Impact Statement}

This work aims to support scalable and reproducible embodied AI research by automatically generating executable robot interaction data in simulation. By producing task configurations, action traces, execution videos, and offline learning datasets, \method~ may reduce the manual effort required to construct embodied training data and make robot-learning experiments easier to inspect and reproduce. Since all demonstrations are generated and refined in simulation, the framework can also help identify unsafe or failed behaviors before real-world deployment. However, the quality and diversity of the generated data depend on the underlying scene reconstruction tools, simulation assets, primitive action library, and VLM feedback quality. We therefore encourage careful validation before using such generated data in real-world robotic systems.

\section{Implementation Details}

\subsection{Simulated Scene Generation Details}
\label{app:Simulated Scene Generation Details}

This section presents the complete process of simulated scene generation described in Section~\ref{sec:Scene-Conditioned Simulation Substrate}, which follows the reconstruction method of~\cite{acdc}. Given a single RGB scene image, the scene generation pipeline reconstructs an interactive simulated scene while preserving semantic affordances and object states. The pipeline is as follows:

\begin{enumerate}
\item\textbf{Image Extraction:} Given a single RGB scene image $X$ and an estimated camera intrinsic matrix $K$, output a set of object representations including object labels, masks, point clouds, etc., for the subsequent creation of simulated scenes. The main methods include observing and extracting objects (using LLM, \textit{GroundedSAM-v2}~\cite{sam2}), followed by depth estimation (\textit{DepthAnything-v2}~\cite{depth_anything_v2}) and extracting point clouds, etc.
\item\textbf{Assets Matching:} Hierarchical search is conducted on the asset dataset \textit{BEHAVIOR-1K}~\cite{behavior1k}. Through the \textit{CLIP} and \textit{DINOv2} models~\cite{clip, dinov2}, the most suitable assets are matched for each object. This step outputs the virtual assets and snapshots corresponding to each cousin in the corresponding direction.
\item\textbf{Simulated Scene Generation:} Improve the details such as asset location determination and scene physical stability (for example, use LLM to determine whether to install objects on the wall). And based on the \textit{OmniGibson} platform~\cite{behavior1k}, simulation scenes are generated, providing physical simulation and rendering including object states, etc.
\end{enumerate}

ManiAgent~\cite{maniagent} directly estimates the 3D coordinates of real objects from RGB images using \textit{Florence-2}, and then employs \textit{AnyGrasp} to estimate the 6-DoF pose of the target position and plan the path. In contrast, our work estimates the 3D coordinates of each object from RGB images while reconstructing the entire simulated environment. It retains advanced scene attributes, such as spatial object layout $S_0$, key semantics, and physical visibility, which offers advantages such as maintaining global physical consistency and preserving the state of complex long-horizon tasks.

\subsection{Implementation Details of Structured Task Generation}
\label{app:task_generation_details}

This section details the five-module structured task-generation workflow shown on the left of Figure~\ref{fig:Graph-Based Task Generation with Self-Evolution}. Module 1 expands the query, Module 2 decomposes it into symbolic subtasks, and Modules 3--5 instantiate the BDDL conditions, object-size adjustments, and executable action flows. These interfaces implement the scene-conditioned task-search abstraction in Section~\ref{sec:object_action_task_graph}.

\subsubsection{Module 1: Semantic Expansion}
\label{app:Semantic Expansion}
The goal of this stage is to ground a potentially vague user query into a precise scene context.

\textbf{Input:}
\begin{itemize}
    \item \textbf{User Keyword ($I$):} This can be a specific object name (e.g., ``refrigerator'', ``cup''), a specific action (e.g., ``pick up the cup''), or a high-level abstract instruction requiring reasoning (e.g., ``put the cup in the refrigerator'', ``clean the room'').
    \item \textbf{Scene Configuration (\textit{scene\_json}):} A JSON file containing the ground-truth metadata of the reconstructed scene, including object IDs, 3D coordinates, and bounding boxes. This is essential for grounding semantic concepts to specific simulation assets (e.g., mapping ``cup'' to the unique ID \textit{glass\_0}).
    \item \textbf{Scene Observation (Optional):} An RGB image of the scene. If provided, a VLM is employed for visual reasoning; otherwise, an LLM processes the textual JSON data.
\end{itemize}

\textbf{Output:}
\begin{itemize}
    \item \textbf{Task Name (\textit{task\_name}):} A strictly formatted string where all objects are replaced by their unique scene IDs (e.g., ``put glass\_0 into cabinet\_0'').
    \item \textbf{Detailed Task Description (\textit{task\_message\_detail}):} A step-by-step narrative of the task logic generated by the model. For example: \textit{``This is a kitchen scene. There is a cup (glass\_0) on top of a cabinet (cabinet\_0). The goal is to place the cup inside the cabinet. The robot must first open the cabinet door, approach the cup, grasp it, lift it up, navigate to the cabinet, and place it inside...''}
\end{itemize}

\subsubsection{Module 2: Subtask Decomposition}
This stage breaks down the long-horizon task into a sequence of atomic simple tasks, bridging the definition in Section~\ref{sec:Symbolic Task Interface} of the main paper.

\textbf{Input:}
The \textit{task\_name} and \textit{task\_message\_detail} from~\ref{app:Semantic Expansion}, along with the \textit{scene\_json}.

\textbf{Output:}
A sequential list of subtask configurations (JSON). For a complex task $\mathcal{T} = \{T_1, T_2, \dots, T_K\}$, each subtask configuration contains:
\begin{itemize}
    \item \textbf{Subtask Name:} The atomic activity (e.g., ``Open Cabinet'', ``Pick up Cup'').
    \item \textbf{Description:} Detailed instructions specific to this subtask.
    \item \textbf{Target Object ID ($o_t$):} The primary object being manipulated.
    \item \textbf{Support Object ID ($o_{s_1},o_{s_2}$):} The two support objects are associated with the state definitions in the task: the first support object $o_{s_1}$ specifies the initial state relative to the target object (e.g., the table supporting the cup), while the second support object $o_{s_2}$ specifies the goal state relative to the target object (e.g., the container receiving the cup), they all extracted from the \textit{scene\_json}.
    \item \textbf{BDDL Category:} The specific object category mapped to the OmniGibson asset library, used for subsequent BDDL file generation.
\end{itemize}

\subsubsection{Subtask Instantiation}
\label{app:Subtask Instantiation}
For each decomposed subtask, Module 3 generates BDDL initial and goal conditions, Module 4 adjusts object size when needed, and Module 5 plans an action flow. The resulting configurations are then evaluated through simulated execution.

\paragraph{Module 3: BDDL Generation (Initial State and Success Detection)}

We utilize the Behavior Domain Definition Language~\cite{behavior1k} (BDDL) to define the initial state ($s_{\text{init}}$) and success criteria ($s_{\text{goal}}$) for the simulation. The LLM generates a BDDL file based on task-specific templates.

This module takes the subtask configuration as input, including the target object and support-object categories, and outputs a standard \texttt{.bddl} file. Here are some example predicates:
    
\begin{itemize}
    \item $s_{\text{init}}$: \textit{inroom(robot, kitchen)}, \textit{ontop(glass\_0, table\_0)}.
    \item $s_{\text{goal}}$: \textit{open(fridge\_0)}, \textit{inside(apple\_0, bowl\_0)}, or \textit{not(ontop(cup\_0, table\_0))}.
\end{itemize}

\paragraph{Module 4: Object Size Adjustment (Physical Feasibility)}

To ensure objects are graspable by the robot's parallel gripper (maximum width $\approx 0.06$m, depending on the robot size parameters), we dynamically adjust object scales based on their bounding box extents defined in \textit{scene\_json}. We classify objects into two types:

\begin{itemize}
    \item \textbf{Type A (Fixed/Articulated Fixtures):} Objects such as refrigerators, cabinets, microwaves, ovens, drawers, dishwashers, and tables. These are operated via handles or are static supports.
    \item \textbf{Type B (Manipulable Items):} Objects such as cups, glasses, apples, bottles, boxes, bowls, and knives. These must fit inside the gripper.
\end{itemize}

For type A, no scaling is applied, and return scale factor $S = 1.0$. For type B, we calculate the minimum dimension on the horizontal plane: $d_{\min} = \min(\text{bbox}_x, \text{bbox}_y)$. If $d_{\min} > 0.06$m, we compute a scaling factor $S$ such that $S \times d_{\min} \approx 0.05$m (ideal grasping width). Otherwise, $S=1.0$.

\paragraph{Module 5: Action Flow Planning with Primitive Action Library}

Module 5 maps each simple subtask $T$ to an executable action flow using a primitive action library. 
Formally, the output action flow is defined as
\begin{equation}
    \pi(T)=\{a_i(\theta_i\mid T)\}_{i=1}^{n},
\end{equation}
where each $a_i$ is an atomic primitive selected from the library, and $\theta_i$ denotes its task-conditioned parameters. 
These parameters may include the target object $o_t$, support objects $o_{s_1},o_{s_2}$, navigation targets, joint limits, or continuous control values such as distances and angles. 
This notation reflects that primitive actions are not restricted to the target object alone, but may also involve support objects or scene-level navigation goals.

We define a library of \textbf{21 atomic primitives} in the \textit{BasicSkill} class, covering navigation, base movement, end-effector (EEF) operations, articulated-object interaction, and other task-conditioned primitives. 
This library serves as the executable interface shared by the initial action-flow planner and the self-evolution stage. 
During self-evolution, the Robotic Task Supervisor can revise the primitive sequence, replace unsuitable primitives, or tune primitive parameters while remaining within the valid action library.

Each primitive action is represented by an integer primitive ID and a structured parameter dictionary. 
The ID denotes the primitive type rather than an object identifier. 
For example, an action can be represented as
\begin{equation}
    (\mathrm{ID}, \mathrm{params}),
\end{equation}
where $\mathrm{ID}$ indexes a primitive such as \textit{NAVIGATE\_TO\_SUPPORT}, \textit{MOVE\_BASE\_FORWARD}, \textit{GRASP}, or \textit{ARTICULATE\_OPEN}, and $\mathrm{params}$ stores the required object references or continuous control values. 
Thus, a format such as \textit{\{``ID": value\}} should be interpreted as a primitive-ID-to-parameter mapping, not as an object-ID mapping.

We categorize the 21 atomic primitives according to their parameter requirements and functional domains. 
This categorization distinguishes between context-aware primitives that rely on the underlying solver for pose determination and parameterized primitives that allow explicit tuning during self-evolution.

\textbf{Parameter-Free Context-Aware Actions.}

These actions do not require explicit numerical parameters. 
They rely on the underlying motion planner, object state, and task context to determine the target pose or interaction target.
\begin{itemize}
    \item \textit{APPROACH}: Move the end-effector to a pre-grasp or pre-interaction approach pose.
    \item \textit{CONVERGE}: Finely align the end-effector with the target handle or object centroid.
    \item \textit{GRASP}: Close the gripper to grasp the target object.
    \item \textit{UNGRASP}: Open the gripper to release the currently held object.
    \item \textit{RETREAT}: Move the end-effector away from the target to a safe retreat pose.
    \item \textit{NAVIGATE\_TO\_TARGET}: Move the robot base near the target object $o_t$.
    \item \textit{NAVIGATE\_TO\_SUPPORT}: Move the robot base near the support object, typically $o_{s_2}$.
\end{itemize}

\textbf{End-Effector Translation.}

These primitives move the end-effector along a specified direction. 
The parameter is a continuous distance value in meters, represented as \textit{\{``ID": distance\_in\_meters\}}.
\begin{itemize}
    \item \textit{LIFT\_EEF\_UP}: Lift the end-effector vertically upward.
    \item \textit{LIFT\_EEF\_DOWN}: Move the end-effector vertically downward.
    \item \textit{MOVE\_EEF\_FORWARD}: Move the end-effector forward.
    \item \textit{MOVE\_EEF\_BACKWARD}: Move the end-effector backward.
    \item \textit{MOVE\_EEF\_LEFT}: Move the end-effector to the left.
    \item \textit{MOVE\_EEF\_RIGHT}: Move the end-effector to the right.
\end{itemize}

\textbf{Mobile Base Translation and Rotation.}

These primitives control the mobile base relative to the robot's current heading. 
For translation primitives, the parameter is a distance in meters, represented as \textit{\{``ID": distance\_in\_meters\}}.
\begin{itemize}
    \item \textit{MOVE\_BASE\_FORWARD}: Move the robot base forward.
    \item \textit{MOVE\_BASE\_BACKWARD}: Move the robot base backward.
    \item \textit{MOVE\_BASE\_LEFT}: Move the robot base left by holonomic sliding.
    \item \textit{MOVE\_BASE\_RIGHT}: Move the robot base right by holonomic sliding.
\end{itemize}

For rotation primitives, the parameter is an angle in degrees, represented as \textit{\{``ID": angle\_in\_degrees\}}.
\begin{itemize}
    \item \textit{TURN\_BASE\_LEFT}: Rotate the robot base counter-clockwise.
    \item \textit{TURN\_BASE\_RIGHT}: Rotate the robot base clockwise.
\end{itemize}

\textbf{Articulated-Object Interaction.}

These primitives interact with articulated objects such as doors, drawers, and refrigerators. 
The parameter specifies normalized joint limits or physical joint limits, represented as \textit{\{``ID": [min\_limit, max\_limit]\}}. 
We also allow \textit{None} as the parameter value, in which case the default solver limits are used.
\begin{itemize}
    \item \textit{ARTICULATE\_OPEN}: Execute an opening motion for a revolute or prismatic joint.
    \item \textit{ARTICULATE\_CLOSE}: Execute a closing motion for a revolute or prismatic joint.
\end{itemize}

\subsection{Implementation Details of Self-Evolution Mechanism}
\label{app:self_evolution_details}

This section provides implementation details of the \textbf{Self-Evolution Mechanism} in Figure~\ref{fig:Graph-Based Task Generation with Self-Evolution}. 
The main text describes the mathematical formulation of the Robotic Safety Inspector and the Robotic Task Supervisor. 
Here, we focus on the practical execution protocol, the VLM prompt logic, and the reset mechanism used during iterative refinement.

\paragraph{Resetting before each evolution attempt.}
A failed execution may irreversibly alter the scene, such as dropping an object, pushing it out of reach, or leaving the robot in an invalid configuration. 
Therefore, before evaluating each revised action flow, we reset the simulator to a saved subtask initial state $\bar{S}_{\mathrm{init}}^{(k)}$, which satisfies the symbolic initial condition $s_{\mathrm{init}}(T_k)$. 
This ensures that different candidate action flows are evaluated from the same subtask boundary state, making the self-evolution process reproducible and preventing later attempts from being affected by previous failed rollouts.

\subsubsection{Visual Data Pre-processing}

For each action step in the current action flow $\hat{\pi}_{\tau}(T_k)$, we capture RGB frames from multiple camera views $\mathcal{C}\subset\{\mathrm{Global},\mathrm{Head},\mathrm{Wrist}\}$. 
The \textit{Global View} provides a static scene overview, the \textit{Head View} captures the robot-centered perspective, and the \textit{Wrist View} focuses on fine-grained manipulation near the end-effector. 
As defined in the main text, each action-level visual sequence is uniformly downsampled to at most 6 frames, and the inspector receives a sliding-window visual context $X_{k,j}^{\mathcal{C}}$ containing the current action and the previous $p_1$ action steps. 
We use $p_1=1$ by default in the main experiments.

\subsubsection{Step-wise Critique by the Robotic Safety Inspector}

The Robotic Safety Inspector verifies each primitive action locally. 
For the $j$-th action in $\hat{\pi}_{\tau}(T_k)$, the inspector receives the action description and the local visual context $X_{k,j}^{\mathcal{C}}$, and outputs a step-wise critique $m_{\tau,j}$. 
The prompt asks the VLM to check three aspects:

\begin{itemize}
    \item \textbf{Motion Verification:} whether the robot movement matches the intended primitive action, such as whether the base reaches the target region or the end-effector moves in the expected direction.
    \item \textbf{State Contrast:} whether the intended state transition occurs by comparing the beginning and ending frames, such as whether the object is actually grasped or placed.
    \item \textbf{Anomaly Detection:} whether unintended collisions, object displacement, failed grasps, or unsafe motions occur during execution.
\end{itemize}

The output $m_{\tau,j}$ is a concise textual diagnosis of the current action step. 
The full critique sequence for the current attempt is denoted as
$m_\tau=\{m_{\tau,j}\}_{j=1}^{|\hat{\pi}_\tau(T_k)|}$, consistent with the notation in the main text.

\subsubsection{Task-Level Revision by the Robotic Task Supervisor}

The Robotic Task Supervisor performs global action-flow revision. 
It receives the current action flow $\hat{\pi}_{\tau}(T_k)$, the step-wise critique sequence $m_\tau$, the complete visual rollout $X_k^{\mathcal{C}}$, the task context $\mathcal{T}$, and the evolution history $\mathcal{H}_{<\tau}$. 
The history follows the definition in the main text:
\[
    \mathcal{H}_{<\tau}
    =
    \{(\hat{\pi}_i(T_k), m_i, \hat{r}_{i+1})\}_{i=0}^{\tau-1},
\]
where each record stores a previously evaluated action flow, its step-wise critiques, and the supervisor's revision reason.

The supervisor prompt enforces three reasoning steps:
\begin{enumerate}
    \item \textbf{Root Cause Analysis:} identify the earliest failed or risky step, rather than only describing the final failure.
    \item \textbf{Action-Flow Revision:} revise the primitive sequence, tune action parameters, or modify navigation waypoints based on the diagnosis.
    \item \textbf{Library Constraint:} ensure that the revised action flow only uses valid primitives from the primitive action library in Appendix~\ref{app:task_generation_details}.
\end{enumerate}

The supervisor outputs a revised action flow $\hat{\pi}_{\tau+1}(T_k)$ and a global textual explanation $\hat{r}_{\tau+1}$. 
After each failed attempt, the history is updated with the evaluated action flow, its critiques, and the revision reason. 
This memory prevents the supervisor from repeatedly trying previously failed configurations.

\subsubsection{Algorithm Summary}

Algorithm~\ref{alg:self_evolution} summarizes the complete self-evolution procedure. 
The key implementation detail is that every candidate action flow is executed after resetting the simulator to the saved initial state $\bar{S}_{\mathrm{init}}^{(k)}$ of the current subtask.

\begin{algorithm}[h]
\caption{Self-Evolution via Visual Feedback}
\label{alg:self_evolution}
\begin{algorithmic}[1]
\STATE {\bfseries Input:} Initial action flow $\hat{\pi}_0(T_k)$, subtask $T_k$, task context $\mathcal{T}$, camera views $\mathcal{C}$, saved subtask initial state $\bar{S}_{\mathrm{init}}^{(k)}$, max iterations $\tau_{\max}$
\STATE Initialize evolution history $\mathcal{H}_{<0} \leftarrow \emptyset$
\FOR{$\tau = 0$ to $\tau_{\max}$}
    \STATE \textbf{Reset} simulator to $\bar{S}_{\mathrm{init}}^{(k)}$
    \STATE \textbf{Execute} $\hat{\pi}_{\tau}(T_k)$ in simulation and capture visual rollout $X_k^{\mathcal{C}}$
    \IF{success condition $s_{\mathrm{goal}}(T_k)$ is satisfied}
        \STATE {\bfseries Output:} $\hat{\pi}_{\tau}(T_k)$
        \STATE \textbf{Terminate}
    \ENDIF
    \STATE Initialize critique sequence $m_{\tau} \leftarrow \emptyset$
    \FOR{each action step $j$ in $\hat{\pi}_{\tau}(T_k)$}
        \STATE Construct sliding-window visual context $X_{k,j}^{\mathcal{C}}$
        \STATE $m_{\tau,j} \leftarrow \mathrm{VLM}_{\mathrm{Inspector}}\!\left(a_j, X_{k,j}^{\mathcal{C}}\right)$
        \STATE $m_{\tau}.\mathrm{append}(m_{\tau,j})$
    \ENDFOR
    \STATE $\left(\hat{\pi}_{\tau+1}(T_k), \hat{r}_{\tau+1}\right) \leftarrow \mathrm{VLM}_{\mathrm{Supervisor}}\!\left(\hat{\pi}_{\tau}(T_k), m_{\tau}, X_k^{\mathcal{C}}, \mathcal{T}, \mathcal{H}_{<\tau}\right)$
    \STATE $\mathcal{H}_{<\tau+1} \leftarrow \mathcal{H}_{<\tau} \cup \left\{\left(\hat{\pi}_{\tau}(T_k), m_{\tau}, \hat{r}_{\tau+1}\right)\right\}$
\ENDFOR
\STATE {\bfseries Output:} Failure
\end{algorithmic}
\end{algorithm}

\section{Additional Experimental Results}
\label{app:more_exp_results}

\subsection{The Details of Three Camera View}
\label{app:The Details of Three Camera View}

In our Self-Evolution mechanism, the VLM receives visual feedback from a set of cameras $I\subset\{\text{Global, Head, Wrist}\}$. Each perspective offers distinct information density and spatial context, contributing differently to the planning and verification process.

\begin{figure*}[htbp]
    \centering
    \newcommand{\taskcell}[1]{\parbox{0.15\textwidth}{\centering \small #1}}
    \setlength{\tabcolsep}{1pt}
    \begin{tabular}{ccccccc}
        \includegraphics[width=0.15\textwidth]{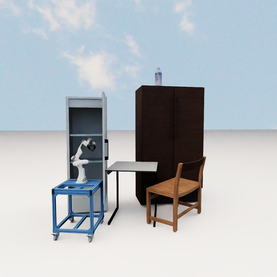} & 
        \includegraphics[width=0.15\textwidth]{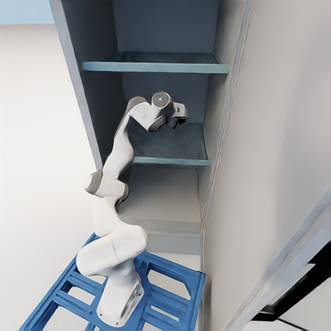} & 
        \includegraphics[width=0.15\textwidth]{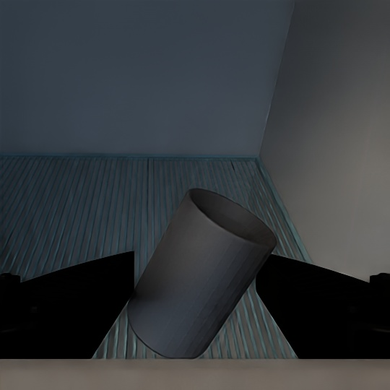} & \quad\  &
        \includegraphics[width=0.15\textwidth]{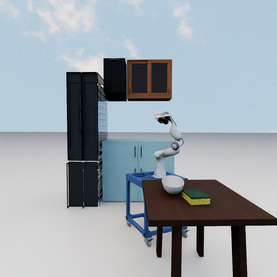} & 
        \includegraphics[width=0.15\textwidth]{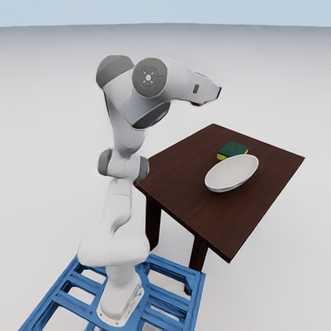} & 
        \includegraphics[width=0.15\textwidth]{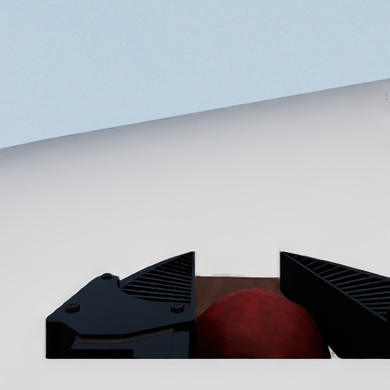} \\
        \taskcell{Global View} & \taskcell{Head View} & \taskcell{Wrist View} & & \taskcell{Global View} & \taskcell{Head View} & \taskcell{Wrist View} \\
    \end{tabular}
    \caption{\textbf{Visual Comparison of Camera Perspectives.} The figure is divided into two task scenarios. \textbf{Left (Columns 1-3):} The task ``Put the cup into the refrigerator''. The \textit{Global View} reveals the arm reaching inside the fridge, though the specific arm posture is occluded. The \textit{Head View} clearly shows the robot's joint state and the open fridge door. The \textit{Wrist View} captures the close-up of the cup within the gripper. \textbf{Right (Columns 4-6):} The task ``Place the apple into the bowl''. The \textit{Global View} shows the overall layout, but depth ambiguity makes it hard to judge if the apple is perfectly over the bowl. The \textit{Head View} compensates for this by showing the alignment relative to the base. However, since the gripper is in mid-air, the \textit{Wrist View} captures mostly empty space, providing little actionable information.}
\label{fig:camera_view_details}
\end{figure*}

\begin{enumerate}
\item \textbf{Global View (Static Scene Overview):} This view corresponds to the original perspective of the input RGB image. Since the generated RGB images lack intrinsic camera parameter matrices, we estimate these parameters during the simulated scene reconstruction process to align the simulation camera with the input image's pose. This view provides the most comprehensive understanding of the task execution, capturing the relative positions of the robot, the target object, and the support surface simultaneously. It is the primary source for checking task completion logic.
\item \textbf{Head View (Robot-Following):} This camera is mounted virtually above the mobile base and moves synchronously with the robot. It serves as an approximation of a ``third-person'' view bound to the agent. This perspective allows for clear observation of the robotic arm's pose, joint configurations, and gripper status (open/closed). However, it lacks global environmental context. As shown in our ablation study Figure \ref{fig:combined_results} (d), relying solely on this view degrades performance in tasks requiring affordance reasoning—for example, it is often impossible to determine the opening direction of a refrigerator door from this angle, leading to navigation failures.

\item \textbf{Wrist View (End-Effector Camera):} This camera is fixed to the robot's wrist and captures the immediate interaction area near the gripper. Its narrow field of view can omit the surrounding geometry needed for navigation and placement. In the $\mathcal{T}_1$ camera-view ablation in Figure~\ref{fig:combined_results}(d), adding the Wrist view to Global and Head provides limited gains. The multi-view settings perform better overall than the evaluated single-view alternatives.
\end{enumerate}

\subsection{Robot System and Execution Details}
\label{app:robot_details}

\paragraph{Robot Configuration and Control.}
We employ a \textit{FrankaMounted} robotic arm equipped with a mobile base for all experiments. The agent perceives the environment through on-board RGB-D sensors. Motion planning is governed by an \textbf{Inverse Kinematics (IK)} controller. For each primitive action defined in the graph, the system first calculates the target coordinate and pose for either the mobile base or the End-Effector (EEF). The IK controller then generates a continuous motion path to reach these targets sequentially. To ensure safe interaction, we initialize the robot at a fixed offset distance relative to the target object before the start of any task execution.

\paragraph{Scale Adaptation.}
A challenge inherent to monocular 3D reconstruction is the ambiguity of absolute scale. Since the scene dimensions are estimated from a single image, the physical size of the reconstructed environment may vary from typical physical dimensions. This discrepancy necessitates a rescaling of the robotic agent to match the simulation's affordances. Through empirical testing, we uniformly set the robot scale factor to $s=0.7$. While this setting works for the majority of our generated scenes, we observed a few failure cases still stemming from scale mismatch.

\paragraph{Grasping Hyperparameters.}
The precision of the \textit{CONVERGE} primitive, which moves the gripper to the final grasping position, is highly sensitive to the distance offset parameter. This offset dictates how far the gripper stops in front of the target coordinate.

If the offset is set too low, the gripper extends too far forward. For tasks like ``Pick up object from table'', this often causes the gripper to collide with the object or grasp it too high, leading to instability. If the offset is set too high, the gripper may stop prematurely, resulting in ``grasping air'' failure where the fingers close without contacting the object.

To address this, we define the convergence offset as a function of the robot's scale. In our experiments, we use a unified hyperparameter: $\delta_\text{offset}=0.14\times s$, where $s=0.7$. This adaptive value ensures that the gripper positioning remains consistent relative to the scaled robot body.

\subsection{Scalable Scene Generation Statistics}
\label{app:Scalable Scene Generation Statistics}

In this section, we provide detailed statistics on the dataset construction, visualize the generated environments, and evaluate the success rates of the autonomously generated atomic tasks. 

To ensure the diversity and realism of our evaluation environment, we constructed a dataset of \textbf{102 unique interactive scenarios}. The generation process follows a hierarchical pipeline:

\begin{enumerate}
    \item \textbf{Source Image Generation:} We used the generative models \textit{Qwen3-Max} and \textit{Seedream-4.5} to synthesize 34 RGB scene images. To cover the most common household manipulation settings, we prompted the model to generate diverse layouts. The distribution includes:
    \begin{itemize}
        \item \textbf{16 Kitchens:} Focusing on complex articulated objects like refrigerators, ovens, microwaves, and cabinets.
        \item \textbf{10 Bedrooms:} Focusing on wardrobes, drawers, and bedside tables.
        \item \textbf{8 Others (Study Rooms, Workshops, etc.):} Focusing on tables, shelves, and miscellaneous clutter.
    \end{itemize}
    \item \textbf{Simulated Scene Reconstruction:} Using the reconstruction pipeline~\cite{acdc}, we extracted semantic and geometric information from these 34 images.
    \item \textbf{Variation Augmentation:} To further expand the dataset, we generated \textbf{3 unique variations} for each RGB scene image by randomizing the top-3 matched assets from the BEHAVIOR-1K dataset (e.g., matching a ``fridge" in the image to three different interactive fridge models in simulation). This resulted in a total of $34 \times 3 = 102$ simulation environments.
\end{enumerate}

\textbf{Automated Task Proposal.}

Unlike manual task design, our framework leverages the \textbf{Structured Task Generation} workflow (detailed in Appendix~\ref{app:task_generation_details}) to autonomously populate these scenes with valid tasks. 

For each of the 102 scenes, we employ a VLM to analyze the spatial layout and propose a context-aware keyword (e.g., ``cup", ``open fridge"). The system then expands this keyword into a full task definition tuple $\langle \text{task\_name}, \text{task\_message}\rangle$. This automated pipeline ensures that the generated tasks are not only semantically aligned with the visual scene but also grounded in the available physical assets. Consequently, we obtained 102 distinct \textit{Scene-Task Pairs}, serving as a reliable benchmark for evaluating the robustness of our method.

\textbf{Visualization of Simulated Scene.} 

We visualize a subset of the generated pairs in Figure~\ref{fig:scene_vis}. The top row displays the RGB scene images generated by Qwen3-Max or Seedream-4.5, while the bottom row shows the reconstructed interactive simulated scenes in OmniGibson. Despite the domain gap, our framework successfully preserves the semantic layout (e.g., the relative position of the fridge and table) and functional affordances (e.g., the graspability of handles), providing a solid foundation for task execution.

\begin{figure*}[t]
    \centering
    \newcommand{\taskcell}[1]{\parbox{0.095\textwidth}{\centering \tiny #1}}
    \setlength{\tabcolsep}{1pt}
    \begin{tabular}{cccccccccc}
        \includegraphics[width=0.095\textwidth]{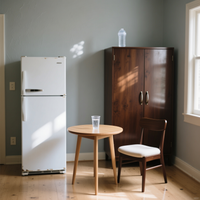} &
        \includegraphics[width=0.095\textwidth]{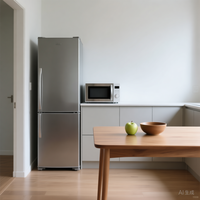} &
        \includegraphics[width=0.095\textwidth]{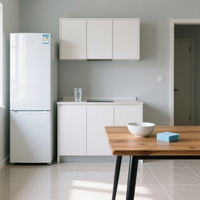} &
        \includegraphics[width=0.095\textwidth]{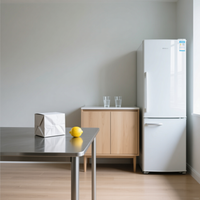} &
        \includegraphics[width=0.095\textwidth]{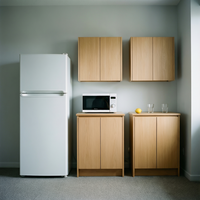} &
        \includegraphics[width=0.095\textwidth]{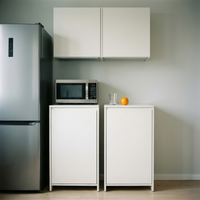} &
        \includegraphics[width=0.095\textwidth]{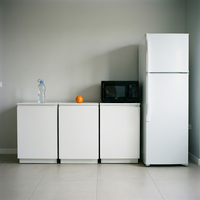} &
        \includegraphics[width=0.095\textwidth]{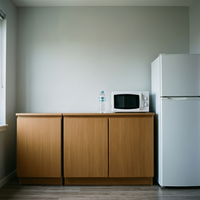} &
        \includegraphics[width=0.095\textwidth]{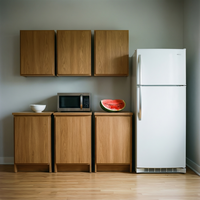} &
        \includegraphics[width=0.095\textwidth]{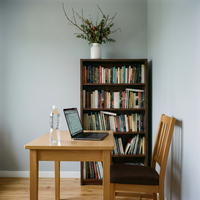} \\
        \includegraphics[width=0.095\textwidth]{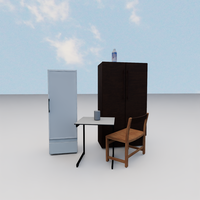} &
        \includegraphics[width=0.095\textwidth]{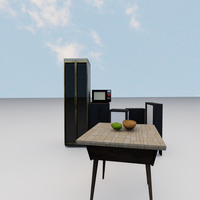} &
        \includegraphics[width=0.095\textwidth]{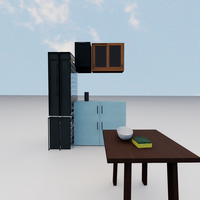} &
        \includegraphics[width=0.095\textwidth]{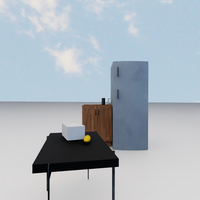} &
        \includegraphics[width=0.095\textwidth]{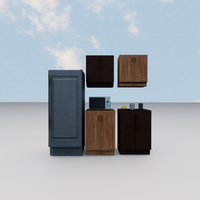} &
        \includegraphics[width=0.095\textwidth]{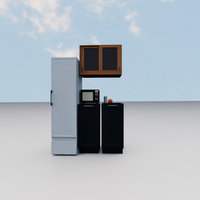} &
        \includegraphics[width=0.095\textwidth]{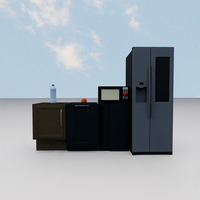} &
        \includegraphics[width=0.095\textwidth]{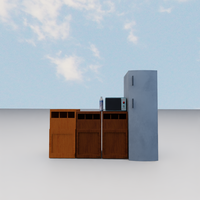} &
        \includegraphics[width=0.095\textwidth]{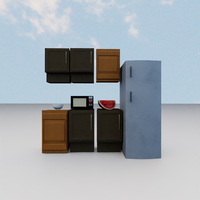} &
        \includegraphics[width=0.095\textwidth]{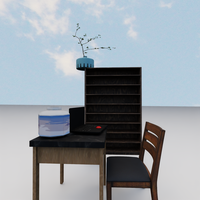} \\
        \taskcell{(1) Open Refrigerator} & 
        \taskcell{(2) Open Cabinet Under The Microwave} & 
        \taskcell{(3) Open The Bottom Cabinet} & 
        \taskcell{(4) Open Cabinet Under Cup} & 
        \taskcell{(5) Open Refrigerator} & 
        \taskcell{(6) Open Cabinet Under The Orange} & 
        \taskcell{(7) Open The Left Cabinet} & 
        \taskcell{(8) Pick Up Water Bottle On The Counter} & 
        \taskcell{(9) Open Cabinet Under The Microwave} & 
        \taskcell{(10) Pick Up Water Bottle On The Table} \\
        \\
        \includegraphics[width=0.095\textwidth]{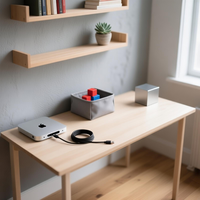} &
        \includegraphics[width=0.095\textwidth]{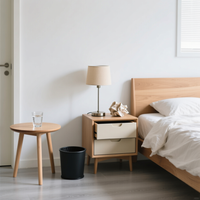} &
        \includegraphics[width=0.095\textwidth]{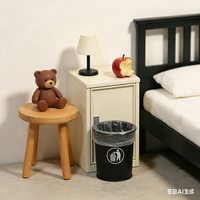} &
        \includegraphics[width=0.095\textwidth]{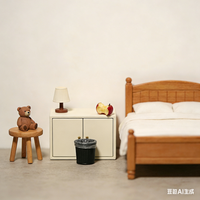} &
        \includegraphics[width=0.095\textwidth]{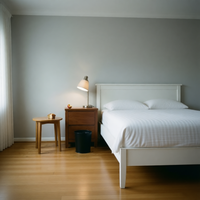} &
        \includegraphics[width=0.095\textwidth]{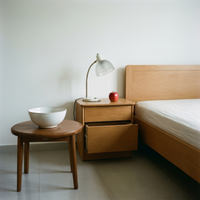} &
        \includegraphics[width=0.095\textwidth]{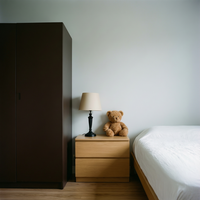} &
        \includegraphics[width=0.095\textwidth]{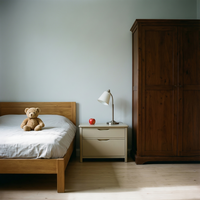} &
        \includegraphics[width=0.095\textwidth]{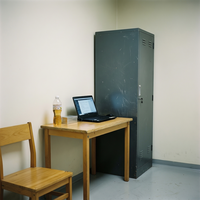} &
        \includegraphics[width=0.095\textwidth]{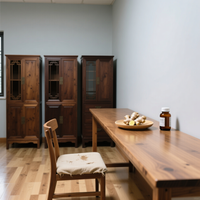} \\
        \includegraphics[width=0.095\textwidth]{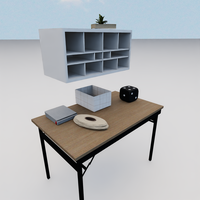} &
        \includegraphics[width=0.095\textwidth]{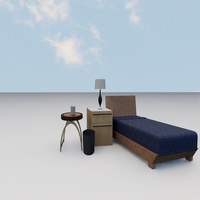} &
        \includegraphics[width=0.095\textwidth]{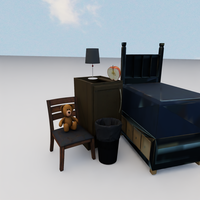} &
        \includegraphics[width=0.095\textwidth]{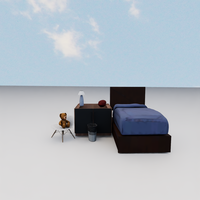} &
        \includegraphics[width=0.095\textwidth]{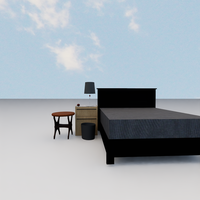} &
        \includegraphics[width=0.095\textwidth]{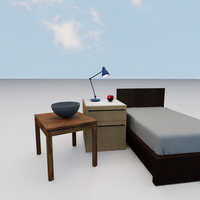} &
        \includegraphics[width=0.095\textwidth]{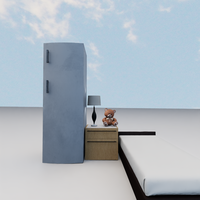} &
        \includegraphics[width=0.095\textwidth]{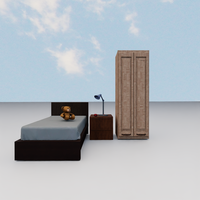} &
        \includegraphics[width=0.095\textwidth]{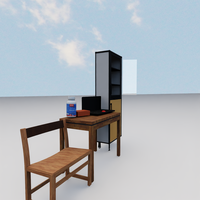} &
        \includegraphics[width=0.095\textwidth]{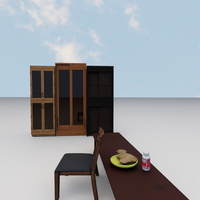} \\
        \taskcell{(11) Pick Up Black Die On The Table} & 
        \taskcell{(12) Pick Up Glass On The Table} & 
        \taskcell{(13) Pick Up Half Apple } & 
        \taskcell{(14) Pick Up Red Apple From Cabinet} & 
        \taskcell{(15) Pick Up Apple On The Table} & 
        \taskcell{(16) Pick Up Apple On The Nightstand} & 
        \taskcell{(17) Pick Up Teddy Bear On The Nightstand} & 
        \taskcell{(18) Pick Up Half Apple From Nightstand} & 
        \taskcell{(19) Pick Up Bottle Of Medicine On The Table} & 
        \taskcell{(20) Pick Up Bottle Of Medicine On The Table} \\
    \end{tabular}
    \caption{\textbf{Visualization of 20 Representative Scene-Task Pairs.} \textbf{Top Row:} RGB scene images generated by Qwen3-Max or Seedream-4.5. \textbf{Bottom Row:} Reconstructed interactive simulation environments. Each column corresponds to a specific atomic task proposed by the VLM, demonstrating our framework's ability to generate diverse household scenarios.}
    \vspace{-0.4cm}
    \label{fig:scene_vis}
\end{figure*}

\textbf{Evaluation of Atomic Task Generation.}

To validate the effectiveness of our structured task-generation workflow, we evaluated the execution success rate on the 102 generated Scene-Task pairs. 

Since the primary goal here is to assess the feasibility of the generated tasks and the reliability of initial action-flow planning, these tasks are primarily simple tasks (atomic primitives) such as ``Open Refrigerator" or ``Pick up Cup". For this experiment, we disabled the \textit{Self-Evolution} mechanism and relied solely on the initial open-loop action flow generated by Module 5.

As shown in Table~\ref{tab:primitive_tasks}, our method achieves an overall success rate of \textbf{71.6\%} across 102 tasks. By analyzing our experimental results, the following conclusions can be drawn:

\begin{itemize}
    \item \textbf{Feasibility:} The observed success rate shows that the generated task definitions and primitive action flows can execute in diverse reconstructed simulations. These experiments evaluate simulation execution, not transfer to physical robots.
    \item \textbf{Data-generation utility:} Successful action configurations can be reused with domain randomization to generate demonstrations for offline policy learning. Figure~\ref{fig:combined_results}(c) reports the generation cost.
    \item \textbf{Role of Self-Evolution:} The failure cases (approx. 28.4\%) were mostly due to edge cases (e.g., kinematic singularities in very narrow corners). This justifies our framework's design: using the structured task generator for the vast majority of data production, while reserving the computationally more expensive \textit{Self-Evolution} mechanism as a specialized agent to solve complex, long-horizon, or failure-prone edge cases.
\end{itemize}

\subsection{Additional Details for Main Task Execution Comparison}
\label{app:Additional Details for Main Task Execution Comparison}

This section provides additional details for the main task execution comparison in Table~\ref{tab:main_comparison}, including the number of evaluated demonstrations or episodes and the definitions of success rate (SR) and Subtask-Level Score. 
We use ``episode'' to denote a distinct evaluated task instance or scene-task configuration, and ``demo'' to denote one rollout under that episode.

\paragraph{Evaluation counts.}
Table~\ref{tab:eval_counts} summarizes the exact number of evaluated demonstrations or episodes for each method and task. 
For $\mathcal{T}_1$--$\mathcal{T}_4$, we evaluate representative long-horizon tasks. 
For $\mathcal{T}_5$ and $\mathcal{T}_6$, we evaluate primitive task categories, where $\mathcal{T}_5$ corresponds to Open/Close tasks and $\mathcal{T}_6$ corresponds to Pick/Place tasks.

\begin{table}[h]
\centering
\caption{
Exact number of evaluated demos and episodes for each method and task. 
Each cell reports all demos as \emph{demos/episode $\times$ episodes}.
}
\label{tab:eval_counts}
\small
\resizebox{0.45\textwidth}{!}{
\begin{tabular}{lcccccc}
\toprule
Method 
& $\mathcal{T}_1$ 
& $\mathcal{T}_2$ 
& $\mathcal{T}_3$ 
& $\mathcal{T}_4$ 
& $\mathcal{T}_5$ 
& $\mathcal{T}_6$ \\
\midrule
Ours (primitive) 
& 10 $(10{\times}1)$ 
& 10 $(10{\times}1)$ 
& 10 $(10{\times}1)$ 
& 10 $(10{\times}1)$ 
& 36 $(1{\times}36)$ 
& 66 $(1{\times}66)$ \\
Ours (prim+Evo) 
& 20 $(20{\times}1)$ 
& 10 $(10{\times}1)$ 
& 10 $(10{\times}1)$ 
& 10 $(10{\times}1)$ 
& -- 
& -- \\
GenSim2 (primitive)
& 20 $(20{\times}1)$ 
& 20 $(20{\times}1)$ 
& 20 $(20{\times}1)$ 
& 20 $(20{\times}1)$ 
& 80 $(20{\times}4)$ 
& 80 $(20{\times}4)$ \\
RoboGen (primitive)
& 10 $(10{\times}1)$ 
& 10 $(10{\times}1)$ 
& 10 $(10{\times}1)$ 
& 10 $(10{\times}1)$ 
& 40 $(10{\times}4)$ 
& 40 $(10{\times}4)$ \\
RoboGen (prim+RL) 
& 1 
& 1 
& 1 
& 1 
& 1 
& 1 \\
\bottomrule
\end{tabular}
}
\end{table}

For GenSim2 and RoboGen, all evaluations are executed using their native automated pipelines. 
For RoboGen (prim+RL), we evaluate one task-specific RL run for each task as a reference solver setting. 
The evaluated RoboGen+RL tasks are: $\mathcal{T}_1$ Put the bottle into the refrigerator, $\mathcal{T}_2$ Put the mouse into the refrigerator, $\mathcal{T}_3$ Put the mouse into the bowl, $\mathcal{T}_4$ Put the bottle on the table, $\mathcal{T}_5$ Open Microwave, and $\mathcal{T}_6$ Pick Bottle.

\paragraph{Success rate.}
For each task, the success rate is computed by averaging over all evaluated rollouts. 
Let $E$ be the number of evaluated episodes and $D_e$ be the number of demos evaluated in episode $e$. 
Then SR is computed as
\begin{equation}
    \mathrm{SR}
    =
    \frac{1}{\sum_{e=1}^{E} D_e}
    \sum_{e=1}^{E}
    \sum_{i=1}^{D_e}
    \mathbf{1}[\mathrm{Success}_{e,i}]
    \times 100\% .
\end{equation}
For our long-horizon tasks $\mathcal{T}_1$--$\mathcal{T}_4$, each task is evaluated in one fixed scene-task episode with multiple demos, e.g., $10{\times}1$ means 10 demos in one episode. 
For our primitive task categories $\mathcal{T}_5$ and $\mathcal{T}_6$, we evaluate many scene-task episodes with one demo each, e.g., $1{\times}36$ for Open/Close and $1{\times}66$ for Pick/Place. 
For complex long-horizon tasks, a demo is counted as successful only if all subtasks are successfully completed.

\paragraph{Subtask-Level Score.}
Besides complete-task SR, we report a Subtask-Level Score to measure partial execution quality. 
Since \method, GenSim2, and RoboGen define and evaluate subtasks differently, this second metric should be interpreted as a diagnostic subtask-level measure rather than a strictly identical mathematical quantity across all systems.

For \method, each long-horizon task is decomposed into a fixed sequence of symbolic subtasks. 
We compute the Subtask-Level Score as the average conditional success rate over these subtasks. 
Let $K$ be the number of subtasks. 
For subtask $T_k$, we compute its success rate among the demos that reach $T_k$:
\begin{equation}
    \mathrm{CondSR}(T_k)
    =
    \frac{
    \sum_i \mathbf{1}[\mathrm{Reach}_{i,k}]
    \mathbf{1}[\mathrm{Success}_{i,k}]
    }{
    \sum_i \mathbf{1}[\mathrm{Reach}_{i,k}]
    } .
\end{equation}
The Subtask-Level Score of \method~ is then
\begin{equation}
    \mathrm{SubtaskScore}^{\method}
    =
    \frac{1}{K}
    \sum_{k=1}^{K}
    \mathrm{CondSR}(T_k)
    \times 100\% .
\end{equation}
This metric measures how reliably each local subtask can be executed once its precondition is reached. 
For our primitive task categories $\mathcal{T}_5$ and $\mathcal{T}_6$, each evaluated task contains one subtask, so the Subtask-Level Score is identical to SR.

For GenSim2, we use its native progress evaluator. 
GenSim2 defines task-specific success criteria and maintains a progress vector over these criteria; the reported score is the average completion value returned by its native evaluator. 
For RoboGen, because the generated task structure is not fixed and may contain both primitive and RL-type subtasks, we compute the ratio of completed generated subtasks to all generated subtasks for each episode and then average across episodes. 
In the primitive-only RoboGen setting, RL subtasks are disabled, so early primitive subtasks may be completed while later subtasks fail, leading to non-zero subtask-level scores even when complete-task SR is $0\%$.

It is worth noting that, although $\mathcal{T}_5$ and $\mathcal{T}_6$ are treated as one-subtask primitive categories in our evaluation, RoboGen may still decompose the corresponding tasks into multiple subtasks. 
Therefore, RoboGen's Subtask-Level Score and SR can differ even on $\mathcal{T}_5$ and $\mathcal{T}_6$.

\subsection{Performance of Example Complex Long-Horizon Tasks}
\label{app:Performance of Example Complex Long-Horizon Tasks}

In this section, we provide a granular breakdown of the generation and instantiation process for the four representative complex long-horizon tasks ($\mathcal{T}_1$ to $\mathcal{T}_4$). We detail the intermediate outputs of our pipeline, from semantic expansion to action flow planning, to demonstrate the reliability and logical consistency of~\method.

\subsubsection{Visual Overview of Examples Scenarios}

To intuitively present the experimental setup, Figure \ref{fig:task_overview_grid} visualizes the four tasks using a structured grid. Each complex long-horizon task displays a generated RGB scene image, simulated reconstruction scene, main target object, main support object, and the user keyword.

\begin{figure*}[t]
    \centering
    \newcommand{\taskcell}[1]{\parbox{0.11\textwidth}{\centering \small #1}}
    \setlength{\tabcolsep}{1pt}
    \begin{tabular}{ccccccccccc}
        \includegraphics[width=0.11\textwidth]{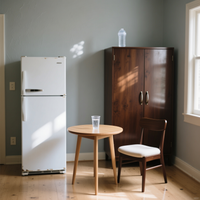} & 
        \includegraphics[width=0.11\textwidth]{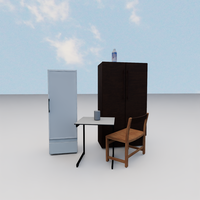} & \quad\  &
        \includegraphics[width=0.11\textwidth]{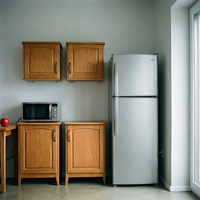} & 
        \includegraphics[width=0.11\textwidth]{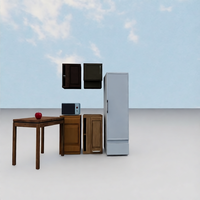} & \quad\  &
        \includegraphics[width=0.11\textwidth]{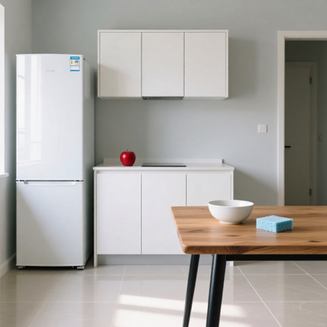} & 
        \includegraphics[width=0.11\textwidth]{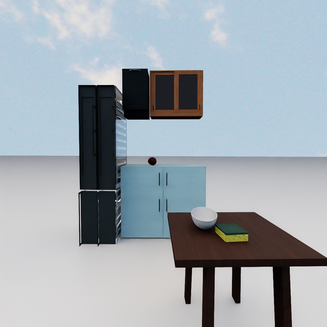} & \quad\  &
        \includegraphics[width=0.11\textwidth]{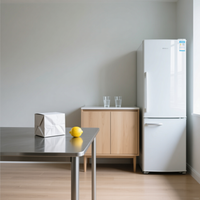} & 
        \includegraphics[width=0.11\textwidth]{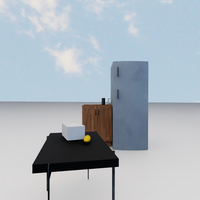}\\
        \includegraphics[width=0.11\textwidth]{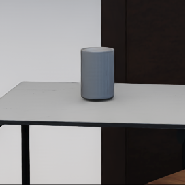} & 
        \includegraphics[width=0.11\textwidth]{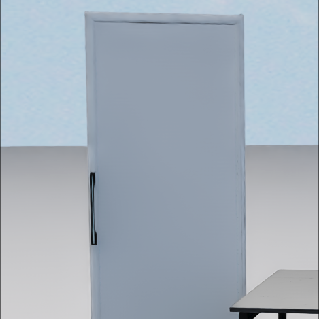} & &
        \includegraphics[width=0.11\textwidth]{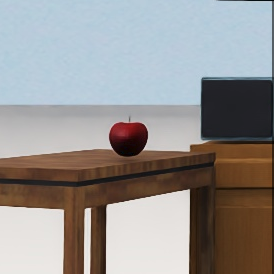} & 
        \includegraphics[width=0.11\textwidth]{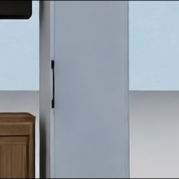} & &
        \includegraphics[width=0.11\textwidth]{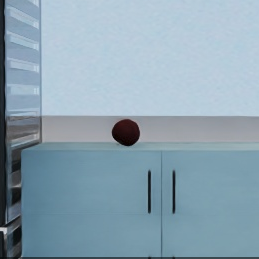} & 
        \includegraphics[width=0.11\textwidth]{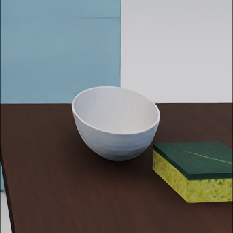} & &
        \includegraphics[width=0.11\textwidth]{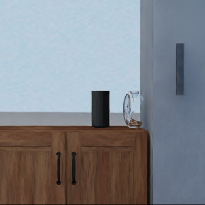} & 
        \includegraphics[width=0.11\textwidth]{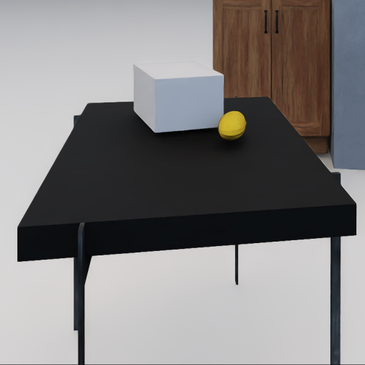}\\
        \taskcell{glass\_0} & \taskcell{refrigerator\_0} & & \taskcell{apple\_0} & \taskcell{refrigerator\_0} & & \taskcell{apple\_0} & \taskcell{bowl\_0} & & \taskcell{glass\_0} & \taskcell{table\_0} \\
        \\
        \multicolumn{2}{c}{\taskcell{``Transport a glass into a refrigerator''}} & & 
        \multicolumn{2}{c}{\taskcell{``Place an apple into a refrigerator''}} & & 
        \multicolumn{2}{c}{\taskcell{``Transfer an apple into a bowl''}} & & 
        \multicolumn{2}{c}{\taskcell{``Place a cup onto a table''}} \\
    \end{tabular}
    \caption{\textbf{Visual Overview of Complex Long-Horizon Tasks.} \textbf{Row 1:} Pairs of generated RGB scene images and reconstructed simulation environments. \textbf{Row 2:} Close-up views of the main target object and support object matched from the simulation environments. \textbf{Row 3:} The user-provided keywords that seed the autonomous task generation process.}
    \label{fig:task_overview_grid}
\end{figure*}

\subsubsection{Semantic Expansion Results}

Upon receiving the user keyword, scene configuration, and observation, the VLM performs Semantic Expansion. While the exact phrasing may vary slightly across different demo runs due to the probabilistic nature of VLMs, the semantic intent remains strictly consistent with the user's instructions. The following are representative outputs for a single demo of each complex long-horizon task:

\textbf{$\mathcal{T}_1$: Put the glass into the refrigerator}

\begin{OutputBox}
"task_activity_name": "open_the_refrigerator_and_put_the_glass_into_the_refrigerator",
"task_detail_message": "This is a kitchen scene, and there is a glass_0 on the table_0, and there is a refrigerator_0 on the left side. The task we need to accomplish is to put the glass_0 on the table_0 into the refrigerator_0. The robot needs to perform a sequence of actions. First, the robotic arm needs to approach the handle of the refrigerator_0 and rotate to open it. Second, pick up the glass_0 from the table_0 and place it inside the refrigerator_0. Finally, close the refrigerator_0."
\end{OutputBox}

\textbf{$\mathcal{T}_2$: Put the apple into the refrigerator}

\begin{OutputBox}
"task_activity_name": "open_the_refrigerator_and_put_the_apple_into_the_refrigerator",
"task_detail_message": "This is a modern kitchen scene. The initial condition is that there is a half_apple (apple_0) resting on top of a breakfast_table (table_0), and a refrigerator (refrigerator_0) is positioned to the right side of the scene, mounted against the wall. The task we need to accomplish is to put the apple_0 into the refrigerator_0. The robot needs to perform a sequence of actions: First, approach the refrigerator_0 and open its door by grasping and rotating the handle. Second, pick up the apple_0 from the table_0. Third, carefully place the apple_0 inside the refrigerator_0. Finally, close the refrigerator_0 door to complete the task."
\end{OutputBox}

\textbf{$\mathcal{T}_3$: Put the apple into the bowl}

\begin{OutputBox}
"task_activity_name": "pick_up_apple_and_put_into_bowl",
"task_detail_message": "This is a modern kitchen scene. The initial condition is that there is an apple_0 resting on the countertop of cabinet_0, and a bowl_0 placed on the wooden table_0 in the foreground. The task we need to accomplish is to pick up the apple_0 from the countertop and place it into the bowl_0. The robotic arm must first approach and grasp the apple_0 while it is stationary on the cabinet_0 surface. Then, without releasing the apple_0, the arm must move it over to the bowl_0 on the table_0 and gently release it inside the bowl_0. No container needs to be opened or closed as both objects are accessible and unobstructed."
\end{OutputBox}

\textbf{$\mathcal{T}_4$: Put the cup on the table}

\begin{OutputBox}
"task_activity_name": "put_the_glass_on_the_table",
"task_detail_message": "This is a modern kitchen scene. The initial condition is that there is a glass_0 (water_glass) currently located near the cabinet_0, not yet on the table_0. The task we need to accomplish is to place the glass_0 onto the table_0. The robotic arm needs to perform the following steps: First, approach and grasp the glass_0 from its current position. Second, lift and move the glass_0 toward the table_0. Third, carefully place the glass_0 onto the surface of the table_0, ensuring it is stable and properly positioned. No other objects need to be moved or manipulated during this task."
\end{OutputBox}

\subsubsection{Subtask Decomposition}

A defining characteristic of~\method~is its structured symbolic decomposition. In the repeated-query study in Figure~\ref{fig:combined_results}(b), the symbolic interface yields a consistent subtask structure, whereas unconstrained baselines produce varying subtask counts. The examples below use this interface to connect object roles with initial and goal predicates. 

Due to space limitations, this document only shows the names and descriptions of each sub-task within the complex long-horizon tasks. Information such as target objects, support objects, and BDDL names, which are included in the output configuration file, are not presented here. The following are representative outputs for a single demo of each complex long-horizon task:

\textbf{$\mathcal{T}_1$: Put the glass into the refrigerator}

\begin{OutputBox}
"sub_task_1": {
    "task_activity_name": "open_refrigerator",
    "task_detail_message": "First, the robotic arm needs to approach the handle of the refrigerator_0 and rotate to open it.",
    ...
},
"sub_task_2": {
    "task_activity_name": "pick_up_glass",
    "task_detail_message": "Second, pick up the glass_0 from the table_0. The robotic arm needs to navigate to the glass_0, approach it, grasp it, and lift it from the table_0.",
    ...
},
"sub_task_3": {
    "task_activity_name": "put_glass_into_refrigerator",
    "task_detail_message": "Third, place the glass_0 inside the refrigerator_0. The robotic arm needs to navigate to the refrigerator_0, move forward while holding the glass_0, and ungrasp to place it inside the refrigerator_0.",
    ...
},
"sub_task_4": {
    "task_activity_name": "close_refrigerator",
    "task_detail_message": "Finally, close the refrigerator_0. The robotic arm needs to approach the handle of the refrigerator_0 and rotate to close it.",
    ...
}
\end{OutputBox}

\textbf{$\mathcal{T}_2$: Put the apple into the refrigerator}

\begin{OutputBox}
"sub_task_1": {
    "task_activity_name": "open_refrigerator",
    "task_detail_message": "First, approach the refrigerator_0 and open its door by grasping and rotating the handle.",
    ...
},
"sub_task_2": {
    "task_activity_name": "pick_up_apple",
    "task_detail_message": "Second, pick up the apple_0 from the table_0. The robotic arm needs to navigate to the target apple_0, approach it, grasp it, and lift it off the table_0.",
    ...
},
"sub_task_3": {
    "task_activity_name": "put_apple_into_refrigerator",
    "task_detail_message": "Third, carefully place the apple_0 inside the refrigerator_0. The robotic arm needs to navigate to the refrigerator_0, move forward, and release the apple_0 inside the refrigerator_0.",
    ...
},
"sub_task_4": {
    "task_activity_name": "close_refrigerator",
    "task_detail_message": "Finally, close the refrigerator_0 door to complete the task. The robot needs to navigate to the refrigerator_0, get close to the handle, and rotate to close the door.",
    ...
}
\end{OutputBox}

\textbf{$\mathcal{T}_3$: Put the apple into the bowl}

\begin{OutputBox}
"sub_task_1": {
    "task_activity_name": "pick_up_apple",
    "task_detail_message": "The robotic arm must first approach and grasp the apple_0 while it is stationary on the countertop of cabinet_0. The target object is apple_0, and the support object is cabinet_0 as the success condition is that the apple_0 is no longer resting on the cabinet_0 surface.",
    ...
},
"sub_task_2": {
    "task_activity_name": "put_apple_into_bowl",
    "task_detail_message": "Then, without releasing the apple_0, the arm must move it over to the bowl_0 on the table_0 and gently release it inside the bowl_0. The target object remains apple_0, but now the support object is bowl_0, as the success condition is that the apple_0 is inside the bowl_0.",
    ...
}
\end{OutputBox}

\textbf{$\mathcal{T}_4$: Put the cup on the table}

\begin{OutputBox}
"sub_task_1": {
    "task_activity_name": "pick_up_glass",
    "task_detail_message": "First, the robotic arm must locate and pick up the glass_0 from its current position (near the cabinet or floor level). The robot needs to approach the glass_0, grasp it securely, and lift it off its current surface.",
    ...
},
"sub_task_2": {
    "task_activity_name": "put_glass_on_table",
    "task_detail_message": "Second, the robotic arm must move the glass_0 toward the table_0 and gently place it on the tabletop, ensuring stability. The robot should navigate to the table_0, lower the glass_0 onto its surface, and release it carefully.",
    ...
}
\end{OutputBox}

\subsubsection{Object Size Adjustment}

To bridge the gap between reconstructed visual assets and physical robot constraints, we implement a dynamic scaling module. For each subtask, the system identifies the target object and applies a scaling factor calculated from its bounding box extents relative to the robot's parallel gripper width ($\approx 0.06m$).

\textbf{Detailed Scaling for $\mathcal{T}_1$ and $\mathcal{T}_2$.} In complex long-horizon task 1, the target object sequence for each sub-task involves [refrigerator\_0: 1.0, glass\_0: 0.37, glass\_0: 0.37, refrigerator\_0: 1.0]. 

For fixed fixtures like the refrigerator, the scale is set to 1.0 to maintain the room's layout. For the manipulable glass\_0 object, a scaling factor of 0.37 was determined in sub-tasks 2 and 3 to ensure the glass diameter was easy to grasp. Finally, we selected the first occurring scaling factors for the refrigerator and the glass, which are 1.0 and 0.37 respectively, and applied these scaling factors.

In complex long-horizon task 2, the target object sequence for each sub-task involves [refrigerator\_0: 1.0, apple\_0: 0.35, apple\_0: 0.4, refrigerator\_0: 1.0]. And we select 1.0 and 0.35 for the refrigerator and the apple respectively

\textbf{Detailed Scaling for $\mathcal{T}_3$ and $\mathcal{T}_4$.} In complex long-horizon task 3 and 4, the target object sequence for each sub-task is [apple\_0: 0.42, apple\_0: 0.42] and [glass\_0: 0.4, glass\_0: 0.4]. So we use the first occurring scaling factors 0.42 and 0.4 for $\mathcal{T}_3$ and $\mathcal{T}_4$ respectively.

\subsubsection{BDDL Generation (Initial State and Success Detection)}

The BDDL generation stage defines the logical boundaries of each subtask. The Table~\ref{tab:Initial and goal states} shows the representative initial and target (goal) states of the BDDL definitions for these four complex long-horizon tasks, and we also show the representative outputs for a single demo.

\begin{table}[h]
\centering
\caption{Initial and goal states for the sub-tasks of our complex long-horizon tasks}
\resizebox{0.45\textwidth}{!}{
\begin{tabular}{c|c|c|c}
\toprule
\textbf{Tasks} & \textbf{Subtask Index} & \textbf{Initial State} & \textbf{Target State} \\
\midrule
\multirow{4}{*}{$\mathcal{T}_1$: Glass$\to$Fridge} & 1 & (not (open refrigerator\_0) & (open refrigerator\_0) \\
 & 2 & (ontop glass\_0 table\_0) & (not (ontop glass\_0 table\_0)) \\
 & 3 & (inside glass\_0 gripper) & (inside glass\_0 refrigerator\_0) \\
 & 4 & (open refrigerator\_0) & (not (open refrigerator\_0)) \\
\midrule
\multirow{4}{*}{$\mathcal{T}_2$: Apple$\to$Fridge} & 1 & (not (open refrigerator\_0) & (open refrigerator\_0) \\
 & 2 & (ontop apple\_0 table\_0) & (not (ontop apple\_0 table\_0)) \\
 & 3 & (inside apple\_0 gripper) & (inside apple\_0 refrigerator\_0) \\
 & 4 & (open refrigerator\_0) & (not (open refrigerator\_0)) \\
\midrule
\multirow{2}{*}{$\mathcal{T}_3$: Apple$\to$Bowl} & 1 & (ontop apple\_0 cabinet\_0) & (not (ontop apple\_0 cabinet\_0)) \\
 & 2 & (inside apple\_0 gripper) & (inside apple\_0 bowl\_0) \\
\midrule
\multirow{2}{*}{$\mathcal{T}_4$: Cup$\to$Table} & 1 & (ontop glass\_0 cabinet\_0) & (not (ontop glass\_0 cabinet\_0)) \\
 & 2 & (inside glass\_0 gripper) & (ontop glass\_0 table\_0) \\
\bottomrule
\end{tabular}
}
\label{tab:Initial and goal states}
\end{table}

\subsubsection{Initial Action Flow}

The mapping of high-level subtask definitions to executable primitive sequences is a critical step in the ~\method~ pipeline. While visual keyframes of these executions are omitted here for brevity, comprehensive video demonstrations of the initial and evolved policies for all tasks are available on our project page (as linked in the Abstract).

Detailed definitions and physical implementations of each atomic primitive (e.g., \textit{APPROACH}, \textit{CONVERGE}, \textit{ARTICULATE}) are provided in Appendix~\ref{app:Subtask Instantiation}. Table~\ref{tab:initial_action_flows} summarizes the representative initial action sequences generated by Module 5 for the four representative complex long-horizon tasks before any self-evolution iterations.

\begin{table*}[h]
\centering
\caption{Initial action flows for the sub-tasks of our representative complex long-horizon tasks}
\resizebox{1.0\textwidth}{!}{
\begin{tabular}{c|c|c}
\toprule
\textbf{Tasks} & \textbf{Subtask} & \textbf{Initial Action Flow (Pre-evolution)} \\
\midrule
\multirow{4}{*}{$\mathcal{T}_1$} & 1 & APPROACH $\to$ CONVERGE $\to$ GRASP $\to$ ARTICULATE\_OPEN $\to$ UNGRASP \\
 & 2 & NAVIGATE\_TO\_TARGET $\to$ APPROACH $\to$ CONVERGE $\to$ GRASP $\to$ LIFT\_EEF\_UP(0.2) \\
 & 3 & NAVIGATE\_TO\_SUPPORT $\to$ MOVE\_BASE\_FORWARD(0.4) $\to$ MOVE\_EEF\_FORWARD(0.1) $\to$ UNGRASP \\
 & 4 & NAVIGATE\_TO\_TARGET $\to$ APPROACH $\to$ CONVERGE $\to$ GRASP $\to$ ARTICULATE\_CLOSE(0.0, 0.5) $\to$ UNGRASP \\
\midrule
\multirow{4}{*}{$\mathcal{T}_2$} & 1 & APPROACH $\to$ CONVERGE $\to$ GRASP $\to$ ARTICULATE\_OPEN(0.0, 0.7) $\to$ UNGRASP $\to$ RETREAT \\
 & 2 & NAVIGATE\_TO\_TARGET $\to$ RETREAT $\to$ APPROACH $\to$ CONVERGE $\to$ GRASP $\to$ LIFT\_EEF\_UP(0.2) \\
 & 3 & NAVIGATE\_TO\_SUPPORT $\to$ MOVE\_BASE\_FORWARD(0.4) $\to$ MOVE\_EEF\_FORWARD(0.1) $\to$ UNGRASP \\
 & 4 & NAVIGATE\_TO\_TARGET $\to$ APPROACH $\to$ CONVERGE $\to$ GRASP $\to$ ARTICULATE\_CLOSE(0.0, 0.7) $\to$ UNGRASP \\
\midrule
\multirow{2}{*}{$\mathcal{T}_3$} & 1 & APPROACH $\to$ CONVERGE $\to$ GRASP $\to$ LIFT\_EEF\_UP(0.2) \\
 & 2 & NAVIGATE\_TO\_SUPPORT $\to$ MOVE\_BASE\_FORWARD(0.4) $\to$ MOVE\_EEF\_FORWARD(0.1) $\to$ LIFT\_EEF\_DOWN(0.3) $\to$ UNGRASP \\
\midrule
\multirow{2}{*}{$\mathcal{T}_4$} & 1 & APPROACH $\to$ CONVERGE $\to$ GRASP $\to$ LIFT\_EEF\_UP(0.2) \\
 & 2 & NAVIGATE\_TO\_SUPPORT $\to$ MOVE\_BASE\_FORWARD(0.4) $\to$ MOVE\_EEF\_FORWARD(0.1) $\to$ LIFT\_EEF\_DOWN(0.3) $\to$ UNGRASP \\
\bottomrule
\end{tabular}
}
\label{tab:initial_action_flows}
\end{table*}

\subsection{Performance of Self-Evolution}
\label{app:Performance of Self-Evolution}

In this section, we present the step-by-step evolution trajectory for the four representative complex long-horizon tasks. The Self-Evolution mechanism plays a crucial role in correcting execution failures caused by geometric uncertainties or kinematic constraints that are difficult to resolve with open-loop planning.

We report the evolutionary process for one representative demonstration of each complex long-horizon task. Tables \ref{tab:t1_evolution}, \ref{tab:t2_evolution}, \ref{tab:t3_evolution}, and \ref{tab:t4_evolution} detail the changes in the action flow across iterations. The ``Iter 0'' row represents the initial open-loop action flow generated by Module 5. Subsequent iterations show how the VLM refined the sequence by adjusting parameters or inserting new primitives.

\paragraph{Evolution of Task $\mathcal{T}_1$ (Put the glass into the refrigerator).}
The evolution process of Task $\mathcal{T}_1$ is shown in Table \ref{tab:t1_evolution}. Subtasks 1 and 2 were successfully completed on the first attempt. However, subtask 3 (place the glass) required two iterations of improvement. Initially, the robot collided with the open refrigerator door, causing the door to close and preventing the subsequent tasks from being executed. Therefore, the VLM adjusted the base movement path. After one iteration, the glass was released too high and too far back, so the VLM adjusted the robot arm's extension range to ensure safe placement in the correct position. Subtask 4 (close the refrigerator) required one iteration of improvement to ensure the refrigerator door was completely closed.

\begin{table*}[h]
\centering
\caption{\textbf{Evolutionary Trajectory for Task $\mathcal{T}_1$: ``Put the glass into the refrigerator''.} Subtasks 1 and 2 succeeded initially. Subtask 3 required 2 iterations to correct placement height and depth. Subtask 4 required 1 iteration to fully close the door.}
\resizebox{1.0\textwidth}{!}{
\begin{tabular}{c|c|c}
\toprule
\textbf{Subtask} & \textbf{Iter} & \textbf{Action Flow} \\
\midrule
1 & 0 & \text{APPROACH} $\to$ \text{CONVERGE} $\to$ \text{GRASP} $\to$ \text{ARTICULATE\_OPEN} $\to$ \text{UNGRASP} \\
\midrule
2 & 0 & \text{NAVIGATE\_TO\_TARGET} $\to$ \text{APPROACH} $\to$ \text{CONVERGE} $\to$ \text{GRASP} $\to$ \text{LIFT\_EEF\_UP}(0.2) \\
\midrule
\multirow{6}{*}{3} & 0 & \text{NAVIGATE\_TO\_SUPPORT} $\to$ \text{MOVE\_BASE\_FORWARD}(0.4) $\to$ \text{MOVE\_EEF\_FORWARD}(0.1) $\to$ \text{UNGRASP} \\
 & & \\
 & 1 & \text{NAVIGATE\_TO\_SUPPORT} $\to$ \text{MOVE\_BASE\_FORWARD}(0.3) $\to$ \text{MOVE\_EEF\_FORWARD}(0.45) $\to$ \text{UNGRASP} \\
 & & \\
 & \multirow{2}{*}{2} & \text{NAVIGATE\_TO\_SUPPORT} $\to$ \text{MOVE\_BASE\_FORWARD}(0.3) $\to$ \text{MOVE\_EEF\_FORWARD}(0.45) $\to$ \text{MOVE\_EEF\_FORWARD}(0.2) \\
 & & $\to$ \text{LIFT\_EEF\_DOWN}(0.3) $\to$ \text{UNGRASP} \\
\midrule
\multirow{4}{*}{4} & 0 & \text{NAVIGATE\_TO\_TARGET} $\to$ \text{APPROACH} $\to$ \text{CONVERGE} $\to$ \text{GRASP} $\to$ \text{ARTICULATE\_CLOSE}(0.0, 0.5) $\to$ \text{UNGRASP} \\
 & & \\
 & \multirow{2}{*}{1} & \text{NAVIGATE\_TO\_TARGET} $\to$ \text{APPROACH} $\to$ \text{CONVERGE} $\to$ \text{GRASP} $\to$ \text{ARTICULATE\_CLOSE}(0.0, 0.6) \\
 & & $\to$ \text{MOVE\_EEF\_FORWARD}(0.1) $\to$ \text{UNGRASP} \\
\bottomrule
\end{tabular}
}
\label{tab:t1_evolution}
\end{table*}

\paragraph{Evolution of Task $\mathcal{T}_2$ (Put the apple into the refrigerator).}
Table \ref{tab:t2_evolution} show the evolution results. Subtask 3 failed initially because the robot could not reach the deep shelf. The evolution process involved significant adjustments to the base position and rotation (\texttt{TURN\_BASE}) over 3 iterations to achieve a valid kinematic configuration for placement.

\begin{table*}[h]
\centering
\caption{\textbf{Evolutionary Trajectory for Task $\mathcal{T}_2$: ``Put the apple into the refrigerator''.} Subtask 3 underwent 3 iterations, introducing base rotation and finer arm control to navigate the apple onto the fridge shelf.}
\resizebox{1.0\textwidth}{!}{
\begin{tabular}{c|c|c}
\toprule
\textbf{Subtask} & \textbf{Iter} & \textbf{Action Flow} \\
\midrule
1 & 0 & \text{APPROACH} $\to$ \text{CONVERGE} $\to$ \text{GRASP} $\to$ \text{ARTICULATE\_OPEN}(0.0, 0.7) $\to$ \text{UNGRASP} $\to$ \text{RETREAT} \\
\midrule
2 & 0 & \text{NAVIGATE\_TO\_TARGET} $\to$ \text{RETREAT} $\to$ \text{APPROACH} $\to$ \text{CONVERGE} $\to$ \text{GRASP} $\to$ \text{LIFT\_EEF\_UP}(0.2) \\
\midrule
\multirow{9}{*}{3} & 0 & \text{NAVIGATE\_TO\_SUPPORT} $\to$ \text{MOVE\_BASE\_FORWARD}(0.4) $\to$ \text{MOVE\_EEF\_FORWARD}(0.1) $\to$ \text{UNGRASP} \\
 & & \\
 & 1 & \text{NAVIGATE\_TO\_SUPPORT} $\to$ \text{MOVE\_BASE\_BACKWARD}(0.3) $\to$ \text{MOVE\_EEF\_FORWARD}(0.4) $\to$ \text{MOVE\_EEF\_FORWARD}(0.2) $\to$ \text{UNGRASP} \\
 & & \\
 & \multirow{2}{*}{2} & \text{NAVIGATE\_TO\_SUPPORT} $\to$ \text{MOVE\_BASE\_BACKWARD}(0.3) $\to$ \text{MOVE\_EEF\_FORWARD}(0.4) $\to$ \text{TURN\_BASE\_LEFT}(45) \\
 & & $\to$ \text{MOVE\_EEF\_FORWARD}(0.2) $\to$ \text{UNGRASP} \\
 & & \\
 & \multirow{2}{*}{3} & \text{NAVIGATE\_TO\_SUPPORT} $\to$ \text{MOVE\_BASE\_BACKWARD}(0.3) $\to$ \text{MOVE\_EEF\_FORWARD}(0.4) $\to$ \text{TURN\_BASE\_LEFT}(45) \\
 & & $\to$ \text{MOVE\_EEF\_FORWARD}(0.3) $\to$ \text{LIFT\_EEF\_DOWN}(0.2) $\to$ \text{UNGRASP} \\
\midrule
4 & 0 & \text{NAVIGATE\_TO\_TARGET} $\to$ \text{APPROACH} $\to$ \text{CONVERGE} $\to$ \text{GRASP} $\to$ \text{ARTICULATE\_CLOSE}(0.0, 0.7) $\to$ \text{UNGRASP} \\
\bottomrule
\end{tabular}
}
\label{tab:t2_evolution}
\end{table*}

\paragraph{Evolution of Task $\mathcal{T}_3$ (Put the apple into the bowl).}
This evolution details is in Table \ref{tab:t3_evolution}, which required precise positioning to ensure the apple landed \textit{inside} the bowl rather than hitting the rim. Subtask 2 evolved twice: first to adjust the approach angle via \texttt{TURN\_BASE}, and second to fine-tune the release height.

\begin{table*}[h]
\centering
\caption{\textbf{Evolutionary Trajectory for Task $\mathcal{T}_3$: ``Put the apple into the bowl''.} The placement subtask (Subtask 2) required 2 iterations to prevent collision with the bowl's rim.}
\resizebox{1.0\textwidth}{!}{
\begin{tabular}{c|c|c}
\toprule
\textbf{Subtask} & \textbf{Iter} & \textbf{Action Flow} \\
\midrule
1 & 0 & \text{APPROACH} $\to$ \text{CONVERGE} $\to$ \text{GRASP} $\to$ \text{LIFT\_EEF\_UP}(0.2) \\
\midrule
\multirow{7}{*}{2} & 0 & \text{NAVIGATE\_TO\_SUPPORT} $\to$ \text{MOVE\_BASE\_FORWARD}(0.4) $\to$ \text{MOVE\_EEF\_FORWARD}(0.1) $\to$ \text{LIFT\_EEF\_DOWN}(0.3) $\to$ \text{UNGRASP} \\
 & & \\
 & \multirow{2}{*}{1} & \text{NAVIGATE\_TO\_SUPPORT} $\to$ \text{MOVE\_BASE\_BACKWARD}(0.15) $\to$ \text{MOVE\_EEF\_FORWARD}(0.4) $\to$ \text{TURN\_BASE\_LEFT}(15) \\
 & & $\to$ \text{MOVE\_EEF\_FORWARD}(0.1) $\to$ \text{LIFT\_EEF\_DOWN}(0.4) $\to$ \text{UNGRASP} \\
 & & \\
 & \multirow{2}{*}{2} & \text{NAVIGATE\_TO\_SUPPORT} $\to$ \text{MOVE\_BASE\_BACKWARD}(0.3) $\to$ \text{MOVE\_EEF\_FORWARD}(0.2) $\to$ \text{TURN\_BASE\_RIGHT}(20) \\
 & & $\to$ \text{MOVE\_EEF\_FORWARD}(0.25) $\to$ \text{LIFT\_EEF\_DOWN}(0.4) $\to$ \text{UNGRASP} \\
\bottomrule
\end{tabular}
}
\label{tab:t3_evolution}
\end{table*}

\paragraph{Evolution of Task $\mathcal{T}_4$ (Put the cup on the table)}
Similar to $\mathcal{T}_3$, Table \ref{tab:t4_evolution} shows that placing a cup on a specific spot on the table required adjusting the robot base's orientation (\texttt{TURN\_BASE}) to reach the target zone.

\begin{table*}[h]
\centering
\caption{\textbf{Evolutionary Trajectory for Task $\mathcal{T}_4$: ``Put the cup on the table''.} Subtask 2 required 2 iterations, heavily relying on base rotation to align with the table surface.}
\resizebox{1.0\textwidth}{!}{
\begin{tabular}{c|c|c}
\toprule
\textbf{Subtask} & \textbf{Iter} & \textbf{Action Flow} \\
\midrule
1 & 0 & \text{APPROACH} $\to$ \text{CONVERGE} $\to$ \text{GRASP} $\to$ \text{LIFT\_EEF\_UP}(0.2) \\
\midrule
\multirow{7}{*}{2} & 0 & \text{NAVIGATE\_TO\_SUPPORT} $\to$ \text{MOVE\_BASE\_FORWARD}(0.4) $\to$ \text{MOVE\_EEF\_FORWARD}(0.1) $\to$ \text{LIFT\_EEF\_DOWN}(0.3) $\to$ \text{UNGRASP} \\
 & & \\
 & \multirow{2}{*}{1} & \text{NAVIGATE\_TO\_SUPPORT} $\to$ \text{TURN\_BASE\_RIGHT}(60) $\to$ \text{MOVE\_EEF\_FORWARD}(0.4) $\to$ \text{MOVE\_EEF\_FORWARD}(0.1) \\
 & & $\to$ \text{LIFT\_EEF\_DOWN}(0.3) $\to$ \text{UNGRASP} \\
 & & \\
 & \multirow{2}{*}{2} & \text{NAVIGATE\_TO\_SUPPORT} $\to$ \text{TURN\_BASE\_RIGHT}(60) $\to$ \text{MOVE\_EEF\_FORWARD}(0.4) $\to$ \text{MOVE\_EEF\_FORWARD}(0.2) \\
 & & $\to$ \text{LIFT\_EEF\_DOWN}(0.3) $\to$ \text{UNGRASP} \\
\bottomrule
\end{tabular}
}
\label{tab:t4_evolution}
\end{table*}

\paragraph{VLM Reasoning Process.}
To illustrate the decision-making capability of the Self-Evolution mechanism, we provide the raw reasoning outputs for Task $\mathcal{T}_1$ in the box below. This includes the step-wise critiques $m_\tau$ and the final Global Supervisor reasoning.

First, subtask 3 underwent two iterations, and the output of all iterations are:

\begin{OutputBox}
"iter_1": {
    "VLM-Observation": {
        "Step 0 observation (4 frames)": "The robot base navigates to the support (refrigerator), positioning itself near the open refrigerator door. The refrigerator door remains open throughout, and no collision or environmental change is observed during this step.",
        "Step 1 observation (4 frames)": "The robot base moves forward 0.45 meters toward the refrigerator. During this motion, the robotic arm collides with the open refrigerator door, causing it to swing shut partially, which alters the environment and obstructs access to the refrigerator compartment.",
        "Step 2 observation (2 frames)": "The robot performs the UNGRASP action while its arm is positioned near the now partially closed refrigerator door. No further collision occurs, but the glass is not placed inside the refrigerator due to the obstructed access caused by the door's position. The environment remains unchanged during this step."
    },
    "reason": "During the execution of 'put_glass_into_refrigerator', the robotic arm collided with the open refrigerator door while performing 'MOVE_BASE_FORWARD', causing the door to close partially and blocking access to the compartment. Therefore, before 'MOVE_BASE_FORWARD', the following step should be added: 'MOVE_BASE_LEFT 0.3m' to clear the door swing path.",
    "new_sequence": [18, {"15": 0.3}, {"13": 0.45}, 5] 
}
\end{OutputBox}

\begin{OutputBox}
"iter_2": {
    "VLM-Observation": {
        "Step 0 observation (4 frames)": "The robot base navigates toward the refrigerator, positioning itself so the arm can reach the open door. No collision occurs, and the refrigerator door remains open; the environment status is unchanged.",
        "Step 1 observation (3 frames)": "The robot base moves left by 0.3 meters, repositioning the arm slightly away from the refrigerator door. No collision occurs, and the refrigerator door remains open; no environmental changes are observed.",
        "Step 2 observation (4 frames)": "The robot base moves forward by 0.45 meters, bringing the arm closer to the refrigerator. No collision occurs, and the refrigerator door remains open; no environmental changes are observed.",
        "Step 3 observation (2 frames)": "The robot opens its gripper to release the glass, but the glass is not visible inside the refrigerator, indicating it was not placed successfully. No collision occurs, and the refrigerator door remains open; no environmental changes are observed."
    },
    "reason": "The task failed because after moving the base forward, the robotic arm did not extend far enough into the refrigerator to place the glass inside before releasing it. The UNGRASP action occurred while the gripper was still outside the refrigerator compartment. Therefore, after 'MOVE_BASE_FORWARD', the following steps should be added: 1. 'MOVE_EEF_FORWARD 0.2m' to extend the arm into the refrigerator, and 2. 'LIFT_EEF_DOWN 0.3m' to lower the glass to a safe placement height inside the refrigerator.",
    "new_sequence": [18, {"15": 0.3}, {"13": 0.45}, {"9": 0.2}, {"8": 0.3}, 5]
}
\end{OutputBox}

Next, here are the output results for subtask 4:

\begin{OutputBox}
"iter_1": {
    "VLM-Observation": {
        "Step 0 observation (5 frames)": "The robot successfully navigates to the refrigerator, positioning itself beside the open door without hitting any obstacles. The environment remains unchanged during this step; the refrigerator door stays open, and no unintended movements or collisions occur.",
        "Step 1 observation (1 frames)": "The robot arm extends toward the refrigerator handle in a controlled manner, maintaining clearance from surrounding objects. No collision occurs, and the refrigerator door remains open without any unintended movement or environmental change.",
        "Step 2 observation (1 frames)": "The robot arm finely adjusts its position to align precisely with the refrigerator handle, showing no contact with surrounding objects. The refrigerator door remains open and stationary, with no environmental changes or collisions observed.",
        "Step 3 observation (1 frames)": "The robot gripper successfully closes around the refrigerator handle without colliding with any surrounding objects. The refrigerator door remains open and stationary, with no unintended environmental changes observed during the grasp.",
        "Step 4 observation (6 frames)": "The robot arm attempts to close the refrigerator door by pulling the handle, but the door does not fully shut and appears to be obstructed or misaligned. No collision with surrounding objects is observed, but the environment status changes partially -- the door moves slightly but fails to reach a fully closed position, indicating an incomplete articulation.",
        "Step 5 observation (2 frames)": "The robot opens its gripper to release the refrigerator handle, with no contact or collision observed. The refrigerator door remains partially open and does not move further during ungrasping, indicating that the failure to fully close occurred prior to this step."
    },
    "reason": "The failure occurred during the 'ARTICULATE_CLOSE' step. Although the robot successfully grasped the handle and attempted to close the door, the door did not fully shut, likely due to insufficient pulling force or misalignment in the articulation range. To correct this, we should extend the articulation range slightly beyond 0.5 and add a small forward motion of the EEF before ungrasping to ensure full closure.",
    "new_sequence": [17, 1, 2, 3, {"19": [0.0, 0.6]}, {"9": 0.1}, 5]
}
\end{OutputBox}

\subsection{Failure Details and Analysis}
\label{app:Failure Details and Analysis}

\subsubsection{Performance Variance across VLM Architectures}

We conducted a comparative study on the self-evolution performance of three VLMs: \textbf{Gemini-3-Flash}, \textbf{GPT-5.2}, and \textbf{Qwen3-VL} (specifically, the open-source model \textit{qwen3-vl-235b-22a-thinking}). While all models demonstrated reasoning capabilities, their evolution strategies exhibited distinct characteristics:

\begin{itemize}
    \item \textbf{Qwen3-VL (Conservative Strategy):} This model tends to be conservative in its modifications. Its proposed changes often involve minimal adjustments, such as tweaking a single parameter or adding just one or two atomic actions. While this ensures stability, Qwen3-VL frequently fails to identify ``stuck'' situations, where the robot is physically blocked but not in collision, resulting in minor updates that do not resolve the root cause.
    
    \item \textbf{Gemini-3-Flash (Aggressive Strategy):} Gemini-3-Flash exhibits a more aggressive and perceptive reasoning style. It is often capable of identifying subtle geometric constraints, such as the robot arm grazing the edge of a fridge door. Consequently, it proposes significant modifications, such as changing a movement parameter from around 0.3m to around 0.8m. However, given the offline nature of our evolution mechanism (where the VLM cannot interactively query state during reasoning), these large-scale parameter jumps sometimes degrade accuracy, causing the robot to deviate from the correct target position.
    
    \item \textbf{GPT-5.2 (Redundant Strategy):} GPT-5.2 shows a tendency to over-complicate solutions. It frequently stacks new actions onto the existing sequence rather than refining the current ones. This often leads to a ``death spiral'' where the action sequence grows to over 10 steps. As the sequence length increases, the probability of cumulative execution error rises, making success increasingly difficult despite valid high-level logic.
\end{itemize}

\subsubsection{Case Studies on Persistent Failures}

Qualitative inspection of failure cases reveals that errors are rarely due to high-level logic. Instead, they stem from low-level kinematic constraints, such as the robotic arm reaching a singularity when confined against the fridge side, or the hard-coded action primitives lacking the reflexivity to compensate for minor grasping errors. 

\textbf{Case 1: Kinematic Obstruction in Task $\mathcal{T}_1$ (Glass $\to$ Fridge).} Figure \ref{fig:failure_t1} shows the first 4 iterations of subtask 3 (placing the glass into the refrigerator), which ultimately failed after 5 iterations. The visual record shows the robot attempting to extend its arm deep into the refrigerator.

\begin{figure*}[htbp]
    \centering
    \newcommand{\taskcell}[1]{\parbox{0.11\textwidth}{\centering \small #1}}
    \setlength{\tabcolsep}{1pt}
    \begin{tabular}{ccccccccccc}
        \includegraphics[width=0.11\textwidth]{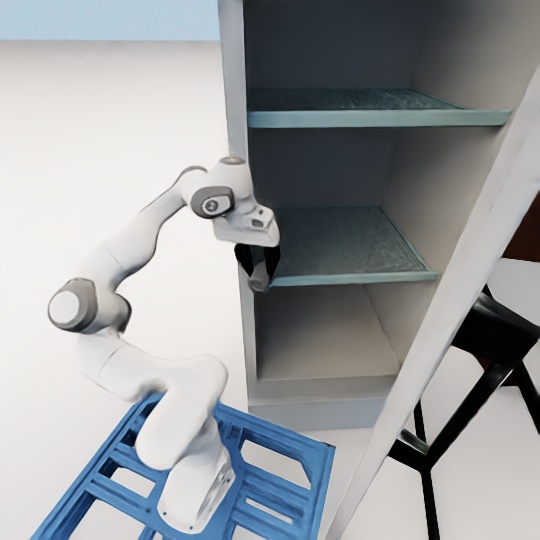} & 
        \includegraphics[width=0.11\textwidth]{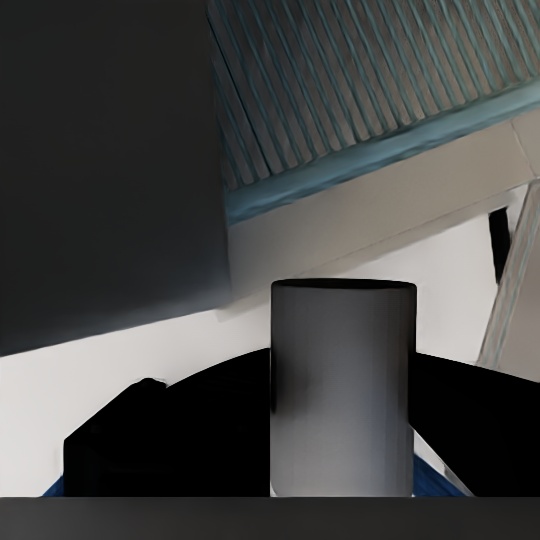} & \quad\  &
        \includegraphics[width=0.11\textwidth]{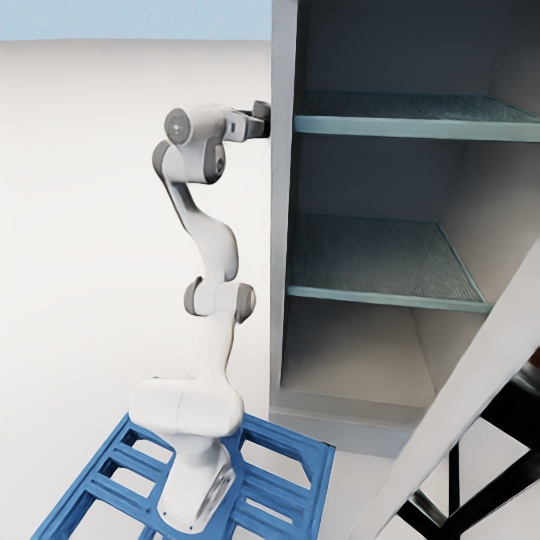} & 
        \includegraphics[width=0.11\textwidth]{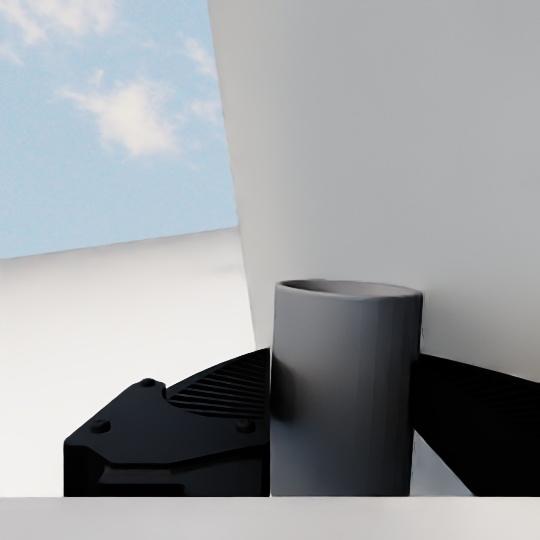} & \quad\  &
        \includegraphics[width=0.11\textwidth]{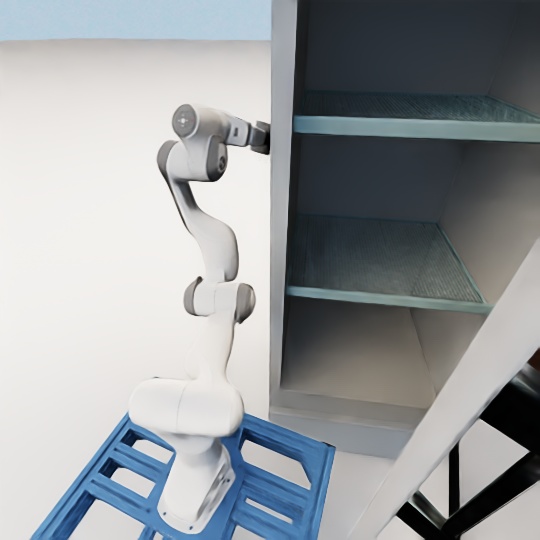} & 
        \includegraphics[width=0.11\textwidth]{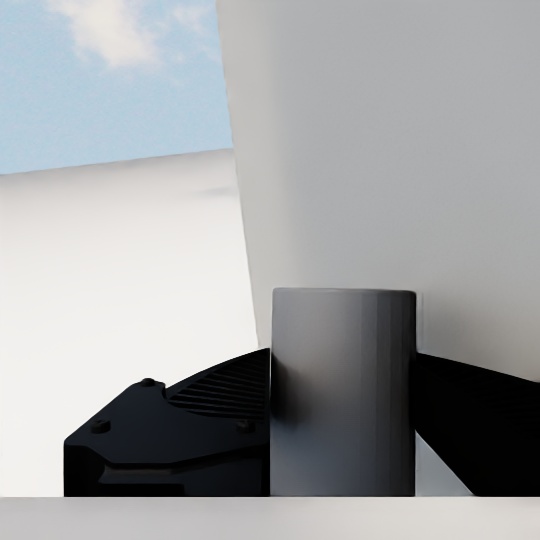} & \quad\  &
        \includegraphics[width=0.11\textwidth]{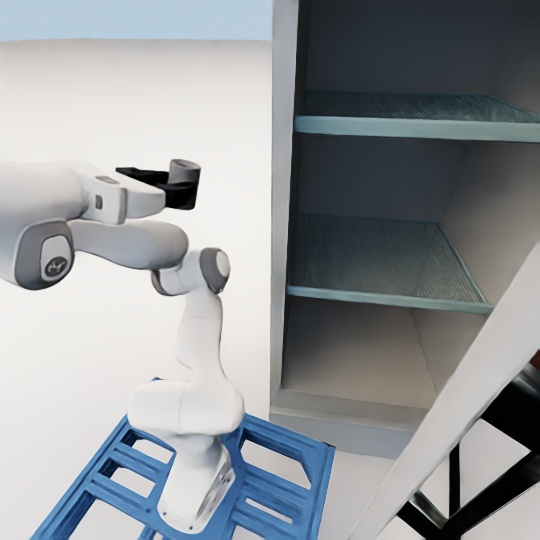} & 
        \includegraphics[width=0.11\textwidth]{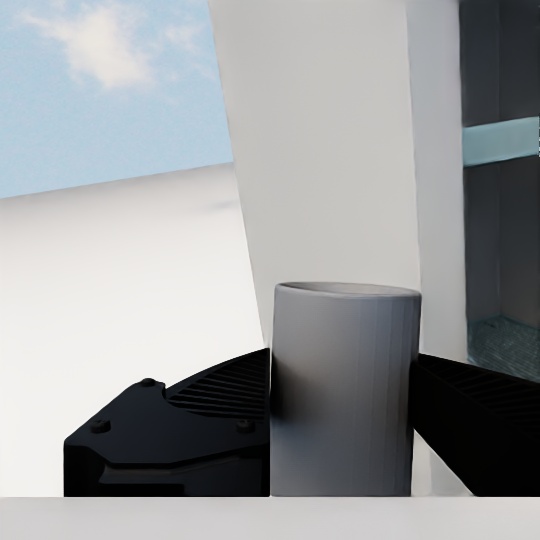}\\
        \multicolumn{2}{c}{\textbf{Iter 1}} & & 
        \multicolumn{2}{c}{\textbf{Iter 2}} & & 
        \multicolumn{2}{c}{\textbf{Iter 3}} & & 
        \multicolumn{2}{c}{\textbf{Iter 4}} \\
    \end{tabular}
    \caption{\textbf{Failure Analysis for $\mathcal{T}_1$ Subtask 3.} Despite 4 rounds of evolution (Iter 5 also failed), the robot consistently fails to insert the glass. The root cause, visible in the Head View, is that the lower link of the robotic arm collides with the left wall of the refrigerator, preventing the end-effector from advancing. This kinematic constraint is difficult for the VLM to diagnose from 2D images, as it requires a subtle backward adjustment of the mobile base to clear the singularity, which is a correction that iterative parameter tuning failed to find.}
\label{fig:failure_t1}
\end{figure*}
As observed in the figures, the root cause is not the arm extension parameter, but the base positioning. The robotic arm link is physically blocked by the left side of the refrigerator frame, preventing the arm from extending fully. This type of collision is subtle and difficult for the VLM to perceive from the standard camera views.

\textbf{Case 2: Precision Placing in Task $\mathcal{T}_3$ (Apple $\to$ Bowl).}
Figure \ref{fig:failure_t3} depicts a failure in Subtask 2 (Place apple into bowl). The challenge here is fine-grained spatial alignment.

\begin{figure*}[htbp]
    \centering
    \newcommand{\taskcell}[1]{\parbox{0.11\textwidth}{\centering \small #1}}
    \setlength{\tabcolsep}{1pt}
    \begin{tabular}{ccccccccccc}
        \includegraphics[width=0.11\textwidth]{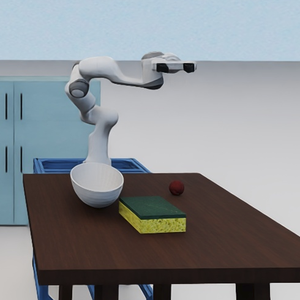} & 
        \includegraphics[width=0.11\textwidth]{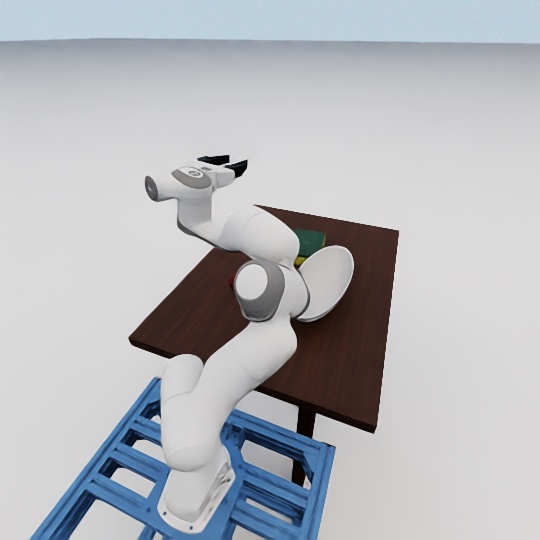} & \quad\  &
        \includegraphics[width=0.11\textwidth]{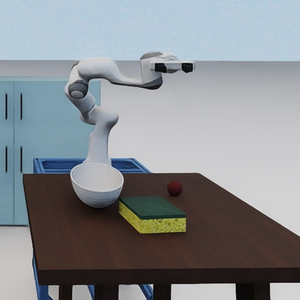} & 
        \includegraphics[width=0.11\textwidth]{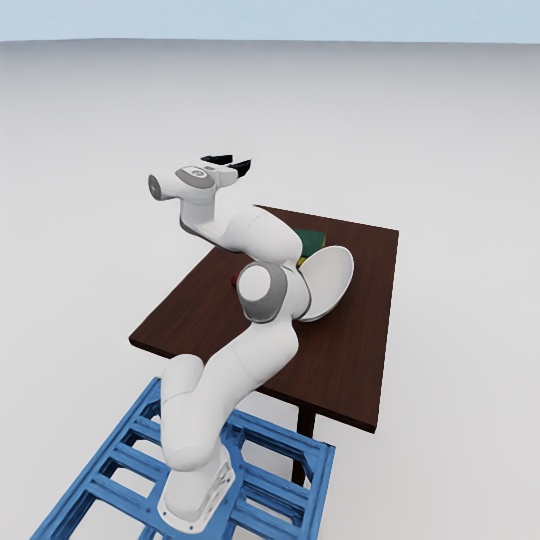} & \quad\  &
        \includegraphics[width=0.11\textwidth]{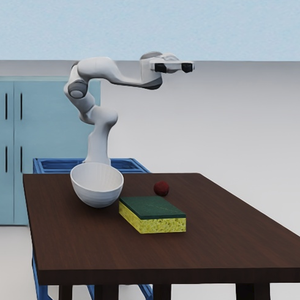} & 
        \includegraphics[width=0.11\textwidth]{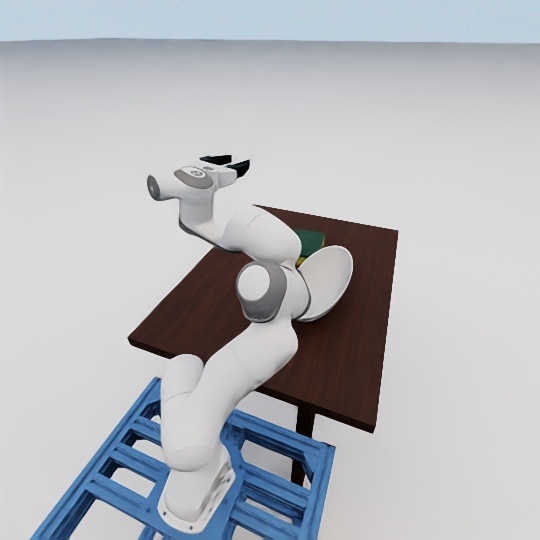} & \quad\  &
        \includegraphics[width=0.11\textwidth]{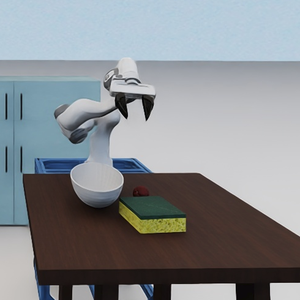} & 
        \includegraphics[width=0.11\textwidth]{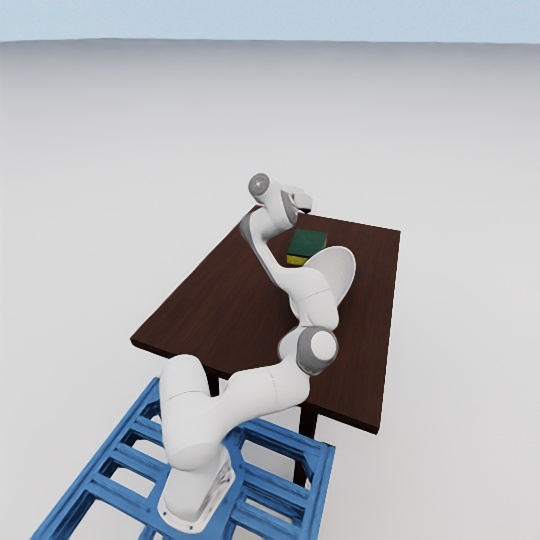}\\
        \multicolumn{2}{c}{\textbf{Iter 1}} & & 
        \multicolumn{2}{c}{\textbf{Iter 2}} & & 
        \multicolumn{2}{c}{\textbf{Iter 3}} & & 
        \multicolumn{2}{c}{\textbf{Iter 4}} \\
    \end{tabular}
    \caption{\textbf{Failure Analysis for $\mathcal{T}_3$ Subtask 2.} The robot struggles to align the apple directly above the bowl. Over 4 iterations, the VLM attempts to correct the drift, but lacks the precision to calculate the exact offset. The open-loop nature means the VLM oscillates between over-correcting (moving too far left) and under-correcting, failing to find the "sweet spot" within the limited iteration budget.}
    \label{fig:failure_t3}
\end{figure*}

As shown in the iteration history, the VLM struggles to calculate the exact spatial offset required to center the gripper over the bowl. Since our method generates the next plan offline based on static visual feedback, the VLM cannot rely on real-time continuous servoing. This results in an oscillation where the robot either adjusts too little (incremental trial-and-error) or modifies parameters too aggressively, missing the target window.

\begin{figure*}[!t]
    \centering
    \includegraphics[width=0.95\linewidth]{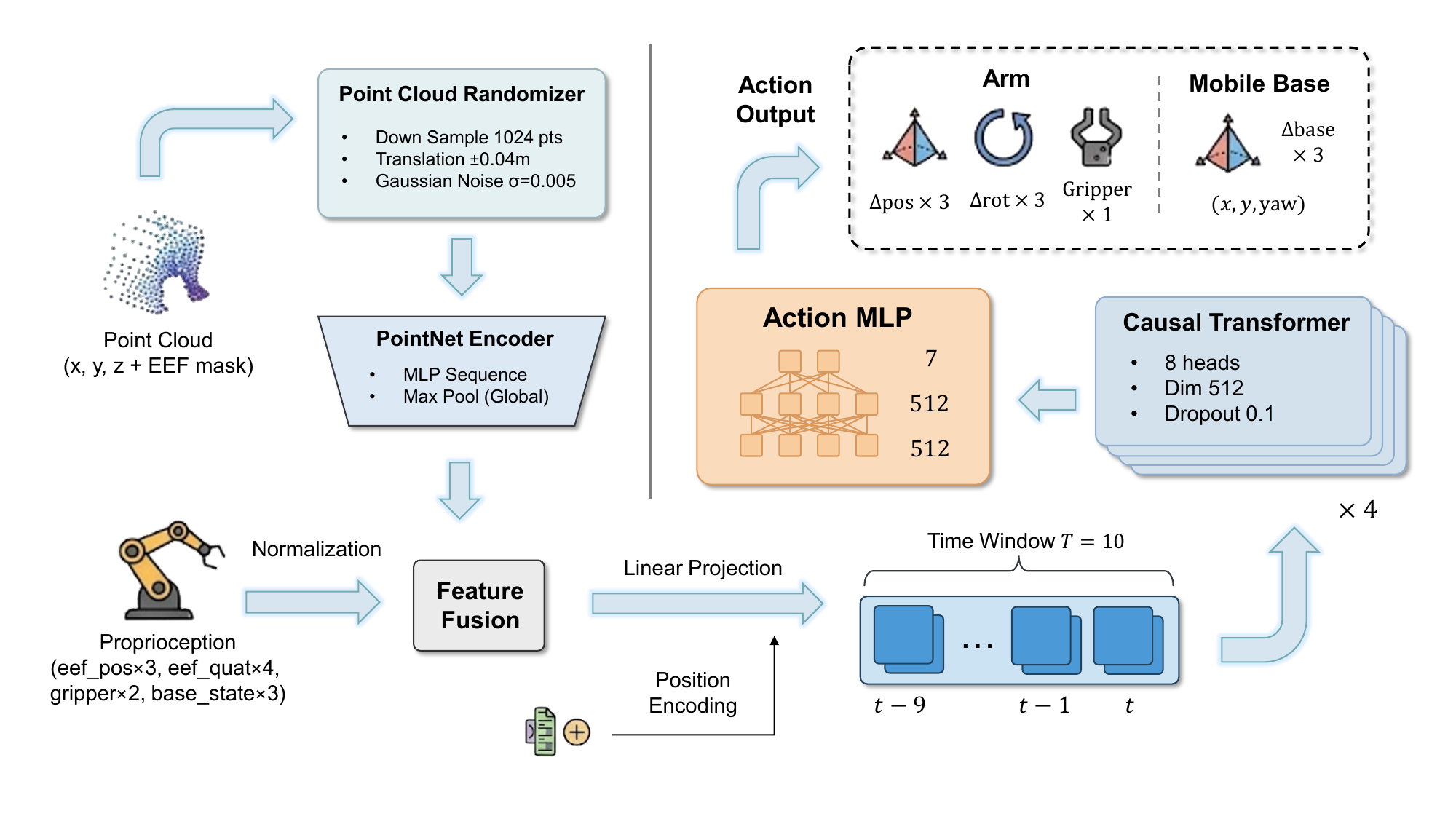}
    \caption{\textbf{Policy architecture for downstream behavior cloning.}
    The policy encodes randomized point clouds with a PointNet encoder, fuses the point-cloud feature with proprioceptive states, and predicts 10-DoF robot actions using a 4-layer causal Transformer and an action MLP.}
    \label{fig:policy_architecture}
\end{figure*}

\subsection{Downstream Policy Learning}
\label{app:Downstream Policy Learning}

We further evaluate whether the demonstrations generated by \method~ can be directly used for downstream policy learning. We train behavior cloning (BC) policies on expert demonstrations autonomously generated by our pipeline. This experiment is not intended to introduce a new policy architecture; instead, it serves as a validation that the generated observation-action data is executable, diverse, and suitable for standard imitation learning.

\paragraph{Policy architecture.}
Figure~\ref{fig:policy_architecture} illustrates the policy architecture. The policy takes point clouds and proprioceptive states as input. The point cloud input contains 3D coordinates $(x,y,z)$ and an end-effector mask. Before encoding, we apply point-cloud randomization, including downsampling to 1024 points, translation perturbation within $\pm 0.04$m, and Gaussian noise with standard deviation $\sigma=0.005$. The randomized point cloud is then encoded by a PointNet encoder composed of an MLP sequence and global max pooling.

The proprioceptive input contains the end-effector position, end-effector quaternion, gripper state, and mobile-base state. We normalize the proprioceptive vector and fuse it with the PointNet feature. The fused feature is linearly projected and combined with positional encoding. We use a causal Transformer with time window $T=10$, 4 layers, hidden dimension 512, 8 attention heads, and dropout 0.1. The Transformer output is passed to an action MLP to predict a 10-DoF action, consisting of 7-DoF arm/gripper control and 3-DoF mobile-base control:
\begin{equation}
\begin{split}
    a_t =&
    [
    \Delta x,\Delta y,\Delta z,
    \Delta r_x,\Delta r_y,\Delta r_z,\\
    &a_{\mathrm{gripper}},
    \Delta x_{\mathrm{base}},
    \Delta y_{\mathrm{base}},
    \Delta \theta_{\mathrm{base}}
    ].
\end{split}
\end{equation}

\paragraph{Training setup.}
Each policy is trained using 100 demonstrations generated by \method. We use one NVIDIA A800 GPU, batch size 1024, and train for 200 epochs. Each epoch samples 102400 training samples. We use the AdamW optimizer with an initial learning rate of $4\times 10^{-4}$. The learning rate follows a linear decay schedule. The BC objective minimizes the prediction error between the policy action and the expert action generated by our pipeline.

\paragraph{Domain randomization.}
We apply domain randomization during data generation to obtain more diverse and robust demonstrations. Specifically, for each generated demonstration, we perturb the target-object scale along the XYZ dimensions with a scale noise of $\pm 0.10$. For tasks involving articulated-object interaction, such as opening a refrigerator, we additionally perturb the target-object orientation around the $z$ axis within $\pm \pi/12$. These perturbations force the policy to learn from varied object geometry and pose configurations instead of overfitting to a single deterministic scene layout.

During policy evaluation, we apply the same domain-randomization setting as in data generation. In particular, the XYZ dimensions of the target objects are perturbed with scale noise $\pm 0.10$, and the target-object rotation around the $z$ axis is perturbed within $\pm \pi/12$ when applicable. This evaluation setting verifies that the learned policies retain robustness under object-scale and pose variations.

\paragraph{Results.}
The generated demonstrations can support effective downstream imitation learning. With 100 generated demonstrations, the BC policy achieves 96.0\% success rate on \emph{Open Refrigerator} with 24 successful trials out of 25, and 92.0\% success rate on \emph{Pick Up Glass} with 23 successful trials out of 25. These results show that \method~ does not only produce visually plausible rollouts, but also generates action-observation trajectories that can be consumed by standard policy-learning pipelines.

% WARNING: do not forget to delete the supplementary pages from your submission 
% \input{sec/X_suppl}

\end{document}